\documentclass[10pt,journal,compsoc]{IEEEtran}
\usepackage{ifpdf}
 \ifpdf
 \else
 \fi

\ifCLASSOPTIONcompsoc
  \usepackage[nocompress]{cite}
\else
  \usepackage{cite}
\fi
\ifCLASSINFOpdf
   \usepackage[pdftex]{graphicx}
\else
   \usepackage[dvips]{graphicx}
\fi
\usepackage[cmex10]{amsmath}
\usepackage{algorithmic}
\ifCLASSOPTIONcompsoc
  \usepackage[caption=false,font=footnotesize,labelfont=sf,textfont=sf]{subfig}
\else
  \usepackage[caption=false,font=footnotesize]{subfig}
\fi
\newcommand{\ignore}[1]{}

\newtheorem{proposition}{Proposition}

\newtheorem{corollary}{Corollary}

\newtheorem{proof}{Proof}

\usepackage{epstopdf}

\usepackage{graphicx}
\usepackage{cite}
\usepackage{color}
\usepackage{soul}

\usepackage{graphics}
\usepackage{amsfonts}
\usepackage{amssymb}
\usepackage{amsmath}
\usepackage{subfig}
\usepackage{comment}
\usepackage{booktabs}
\usepackage{subfig}
\usepackage{listings}
\usepackage{url}
\usepackage{mdframed}
\usepackage{epstopdf}
 \usepackage{gensymb}

\begin{document}
%
\title{OCRAPOSE II: An OCR-based indoor positioning system using mobile phone images}

%
%
%
%

\author{Hamed~Sadeghi, Shahrokh~Valaee~and~Shahram~Shirani


\thanks{H. Sadeghi and S. Valaee are with the Department
of Electrical and Computer Engineering, University of Toronto, ON, Canada, M5S 2E4.\protect\\
E-mail: hsadeghi,valaee@ece.utoronto.ca.}


\thanks{S. Shirani is with the Department of Electrical and Computer Engineering, McMaster University, ON, Canada, L8S 4L8.
Email: shirani@mcmaster.ca.}}

\IEEEtitleabstractindextext{
\begin{abstract}
In this paper, we propose an OCR (optical character recognition)-based localization system called OCRAPOSE II, which is applicable in a number of indoor scenarios including office buildings, parkings, airports, grocery stores, etc. In these scenarios, characters (i.e. texts or numbers) can be used as suitable distinctive landmarks for localization. The proposed system takes advantage of OCR to read these characters in the query still images and provides a rough location estimate using a floor plan. Then, it finds depth and angle-of-view of the query using the information provided by the OCR engine in order to refine the location estimate.
We derive novel formulas for the query angle-of-view and depth estimation using image line segments and the OCR box information. 
We demonstrate the applicability and effectiveness of the proposed system through experiments in indoor scenarios. It is shown that our system demonstrates better performance compared to the state-of-the-art benchmarks in terms of location recognition rate and average localization error specially under sparse database condition.
\end{abstract}

}

\maketitle

\IEEEdisplaynontitleabstractindextext

%
\IEEEpeerreviewmaketitle

%
%
%
%

%
%


\IEEEraisesectionheading{\section{Introduction}\label{sec:introduction}}

\IEEEPARstart{T}{he} most prevalent method for indoor localization is based on fusion of Wi-Fi RSS fingerprints with inertial sensors data \cite{pami1} \ignore{,RSS+inertial,RSS+inertial2}. These methods require a fair number of Wi-Fi access points to be visible at each location \cite{chen}. Under these conditions, they demonstrate localization errors about 2 meters, where training points have granularity of $1 $ meter \cite{chen2}.
In scenarios where enough Wi-Fi access points are not available; access to Wi-Fi RSS reader hardware is blocked (such as iPhones), or when greater localization accuracy (ex. sub-meter accuracy) is required, image-based methods can be used as an effective solution.

Although image-based localization has been studied for a long time in the fields of robotics \cite{pami2,sadeghi} and augmented reality \cite{paucher10,pami3}, it has only been pursued over the past decade for mobile phones in indoor scenarios \cite{sextant}. 
The proposed methods can be categorized into two classes \cite{categorize}, image retrieval-based (fingerprinting-based) \cite{Liang2013,Hamed2} and landmark-based (e.g. logo-based) \cite{Kim10,Hamed1}. 

Both categories of image-based localization methods, i.e. landmark-based and image retrieval-based, require a database of images as well as 3D coordinates/locations  to be measured and stored in the training phase. The data gathering is labor intensive, which is not always possible. 
Furthermore, most of the methods proposed in the literature, utilize \emph{feature extraction and matching} for localization \cite{Liang2013, Hamed1}, while feature extraction and corresponding database creation for a large environment is highly time consuming. Moreover, simple feature matching is not robust to large changes in angle-of-view (AOV). That is, almost the same scene (i.e. set of objects) present in two images cannot be acceptably matched if the difference in AOV is large \cite{MSER, wide1}. Methods such as ASIFT (affine SIFT) \cite{ASIFT}, which are robust to AOV changes have high complexity. For instance, the complexity of ASIFT is at least 1.5 times of that of SIFT \cite{ASIFT2}.
More importantly, we have observed that in a number of indoor environments such as office buildings, parkings, airports, grocery stores, etc., where the distinctive landmarks are text and/or numbers (characters in general), the aforementioned methods fail to provide good location recognition performance for a considerable percentage of queries. 

In scenarios with large image databases, stereo feature matching is computationally expensive and is not performed for best matches detection in the image retrieval-based methods. Instead, bag of features-based methods are used to find the best matching landmark/image(s). 
Here, we use stereo feature matching to illustrate the existence of confusing similar features and lack of distinctive features as the main reasons of failure in best match detection (retrieval).
 Fig. \ref{fig:fail} depicts a university building scenario, where existing literature methods demonstrate poor localization performance.
\begin{figure}[!t]
\centering
\setlength{\fboxsep}{0pt}
\setlength{\fboxrule}{0pt}
\subfloat[]{\label{fig:logofails}\includegraphics
[width=0.45\textwidth,height=0.35\textheight]{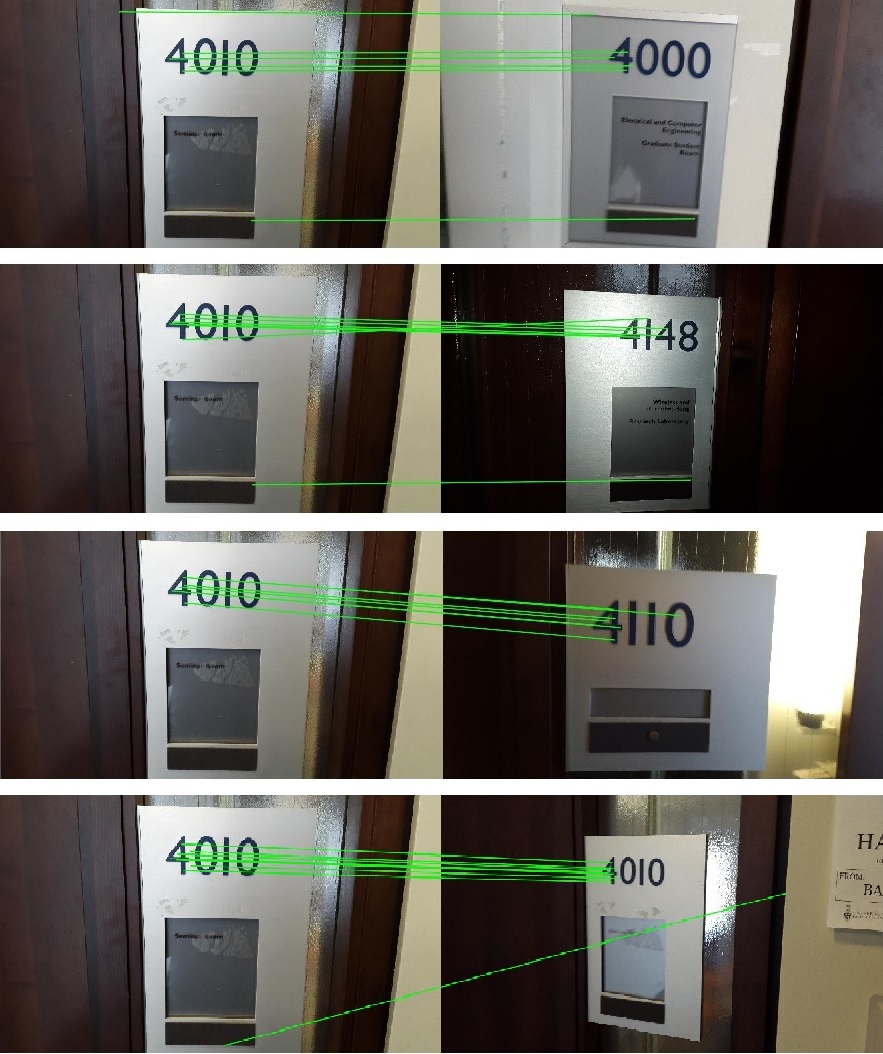}}\\
\subfloat[]{\label{fig:CBIRfails}\fbox{\includegraphics
[width=0.45\textwidth,height=0.4\textheight]{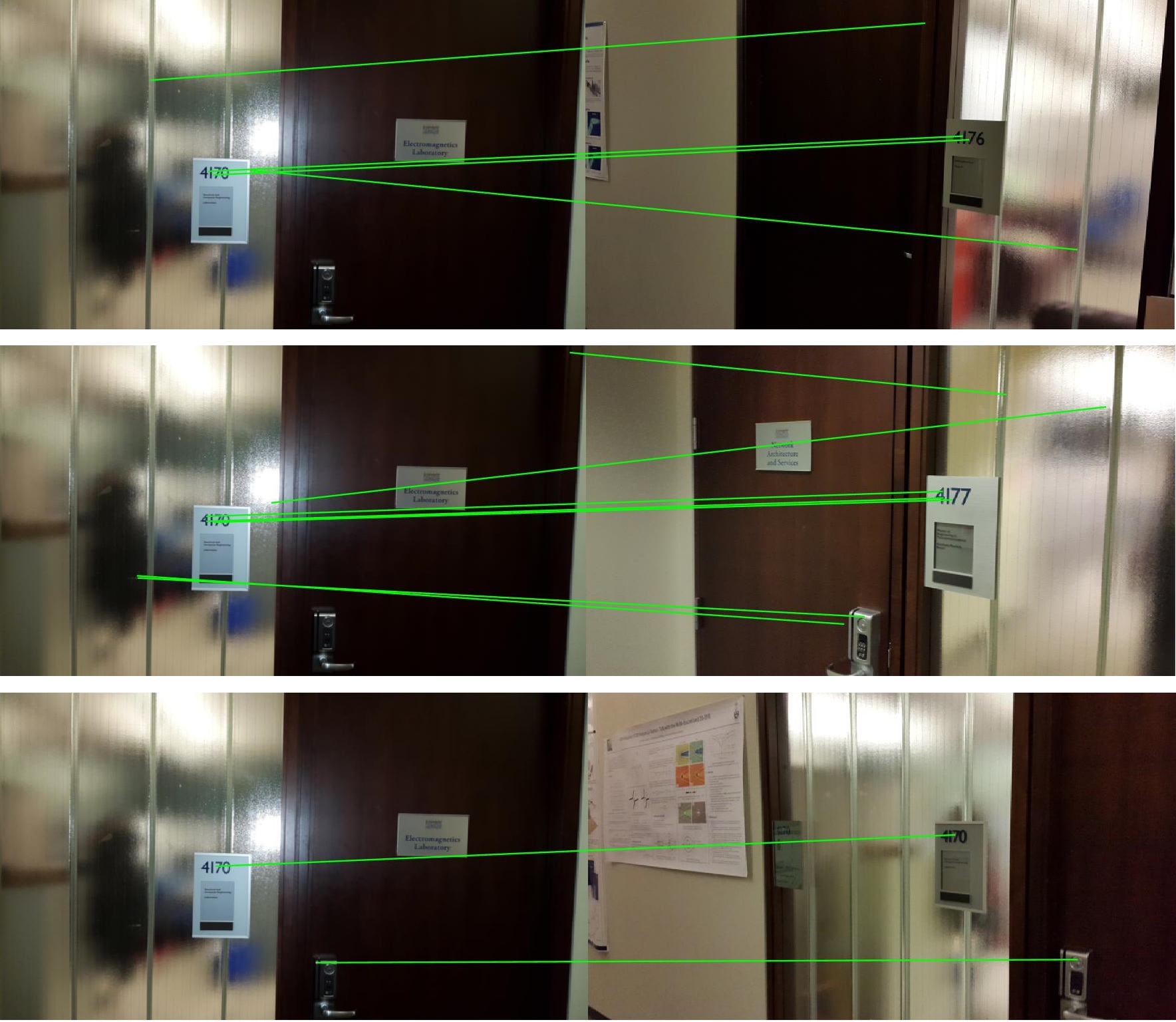}}}
\caption{Sample scenarios, where (a) landmark-based and (b) image retrieval-based methods fail to find the correct match due to the lack of enough distinctive features; the left-hand images are the same query images, the right-hand ones are different database images containing different numbers}
\label{fig:fail} 
\end{figure}

The reason of failure in detecting the correct existing characters using the landmark-based methods is that the characters (numbers) are not as textured as commonly used landmarks such as commercial logos or fiducial markers \cite{Kim10}. Hence, point feature-based recognition approaches \cite{visibility} fail to extract enough distinctive features required to distinguish different numbers from each other. For instance, Fig. \ref{fig:logofails} illustrates why stereo feature matching cannot distinguish which one of the existing database numbers, i.e. $4000$, $4148$, $4110$ or $4010$ corresponds to the query number ($4010$). As seen, although digit $4$ is common among all different numbers, most of features are concentrated around it. This confuses the feature-based recognizer and it cannot find the correct one, i.e. the last bottom one.
 The same problem exists when we want to distinguish between gate numbers in airports, for instance between gate numbers  $\fbox{B42}$ and $\fbox{B43}$. 

Furthermore, there is also a high probability of seeing similar (repeated) scenery from different locations in the mentioned scenarios. Hence, image retrieval-based methods might frequently fail to find the best matching image, i.e. the closest location to the query.
As seen in Fig. \ref{fig:CBIRfails}, the last bottom image, which is the correct database image containing the same numbers ($4170$) has achieved the minimum number of common matching features due to the huge difference in the angle-of-view. As seen, for the selected matching threshold, the door handle has got matched as well as the digit $1$.
On the other hand, although the middle right image contains the number $4177$, it has gained the maximum number of common matches with the query due to a angle-of-view similar to that of the query.

Other image retrieval techniques such as bag of features or localization-specific ones \cite{Liang2013} might also fail to retrieve the correct matching characters. This is due to the existence of more similar (confusing) objects (features) rather than distinctive ones. Doors, windows, door handles, etc are among such objects in our scenario (Fig. \ref{fig:CBIRfails}). We evaluate the recognition performance of the Liang's method \cite{Liang2013} as an example in our experiments.
It should be noted that the the image-retrieval based methods mentioned here use general point feature-based retrieval techniques or retrieval methods designed for location recognition. Use of point feature-based techniques proposed originally for character recognition such as \cite{masa11} will result in an OCR-based localization method, which resides in the same category as our method.

The mentioned issues, i.e. labor extensiveness and poor recognition performance of the conventional point feature-based methods, motivate the use of OCR to recognize the existing location-distinctive characters. By recognizing these characters in the query image, OCR can provide a rough location estimate using a building floor plan.
We assume a floor plan tagged with location-distinctive characters and their locations. An example is seen in Fig. \ref{fig:SF4_plan}. We do not use non location-distinctive characters shared at different locations such as exit signs that potentially confuse the localizer. Furthermore, we only rely on stable characters such as room/gate numbers rather than the texts/numbers on the bulletin boards, which could be replaced or removed after a while. Besides location recognition, OCR provides some clues for fine localization as will be demonstrated later. Hence, the proposed image-based system performs location recognition and location estimation based on OCR.


We do not collect a location-tagged image database for localization. Our system only requires a floor plan tagged with characters locations and the width of the OCR boxes, which can be measured or determined based on typical values. Furthermore, it does not require any feature detection, extraction or matching.
Another characteristic of the proposed localization system is that it can be integrated with other feature-based systems, either landmark-based or image retrieval-based. In \cite{Hamed3}, we demonstrated how OCR can improve a landmark-based localization system in terms of location recognition and localization accuracy.
It should be noted that we are not proposing a character detection/recognition algorithm here. Instead, we focus on OCR-based location recognition as well as the fine localization of user using two novel formulas for depth and angle-of-view estimation. 
Our contributions are three-fold
\begin{itemize}
\item An indoor localization system, called OCRAPOSE II, based on OCR

\item A novel  formula for query angle-of-view estimation

\item A novel formula for query depth estimation

\end{itemize}

Taking advantage of OCR for rough localization and utilizing the proposed formulas to refine the estimate, we provide better localization results in two sample scenarios. The OCRAPOSE II results are better than that of the state-of-the-art in terms of location recognition rate and mean localization error. Furthermore, it shown that the proposed system localization performance is maintained under sparse database locations condition, while that of the benchmarks degrades significantly.

This paper is organized as follows. Section \ref{sec:related} briefly introduces the related literature work and explains the benchmarks in details. Section \ref{sec:system} introduces the proposed system and explains its details.
Section \ref{sec:exp} compares the performance of proposed system with the benchmarks through extensive experiments in two scenarios. Finally, Section \ref{sec:conclusion} concludes the paper.

\section{Related work}
\label{sec:related}

Although image-based robot localization and augmented reality applications have rich literature, image-based user localization has been an active field of research only over the last decade. As explained, the proposed image-based methods can be categorized in two classes. Here, we provide some examples from each category and explain a few methods related works as well as the selected benchmarks in more details. 

Image retrieval-based methods can be applied in general scenarios. These methods perform very well in applications with abundant distinctive landmarks such as shopping malls \cite{Liang2013} or scenarios with unrepeated scenery such as outdoors (using Google street view database) \cite{Torii}. 
They work on the basis that visually different images are taken from different positions. That is, the captured image works as the fingerprint of its (camera) location. Due to this implicit assumption, these methods give poor location estimates in  environments with repeated scenery such as office buildings, parkings, airports, etc. Moreover, a large database of images is required to be collected and Geo-tagged (pose-tagged \cite{Hamed2}) in the training phase. Furthermore, in the test phase, the query image is compared with the entire (or part of) database to find the best match \cite{Liang2013} or best pair \cite{Torii}. Hence, a considerable processing time is required to perform localization during  both training and query phases.

Landmark-based methods are applicable in scenarios, where \emph{highly textured} and \emph{distinctive} landmarks are present \cite{Liang2013, Hamed2}.  
These methods use logo/landmark detection techniques \cite{LOGO_TIP,LOGO} to detect existing landmarks/logos in the query image. Afterwards, the previously measured 3D coordinates of the detected landmark image is used to estimate the camera (full) matrix. The query location can be obtained afterwards from the estimated camera matrix. 
For instance, in shopping malls, where textured commercial logos are ubiquitous, these methods can provide great localization accuracy \cite{Liang2013}. Texture is usually required to extract feature points necessary for unique logo detection. In terms of database size, in contrary to image retrieval-based methods, landmark-based methods require a small image database of existing landmarks as well as some of their actual 3D coordinates. Therefore, their running time is much less than the image retrieval-based methods.


Among a few works proposed in the literature that utilize OCR for localization, \cite{Orlosky14} can be mentioned for its explicit use of OCR as a rough localizer. In fact, it proposes an OCR-assisted multi sensor method for navigation in emergency indoor scenarios. 
The proposed method only provides rough location estimates and no fine localization method has  been proposed.

A good indoor localization system that we use as one of our benchmarks is proposed in \cite{Liang2013} that uses CBIR (content-based image retrieval) to perform rough localization. It requires a location-tagged database. In the training phase, SIFT features of all images are extracted and loaded into a single FLANN kd-tree. In the test phase, SIFT features of the query image are extracted and each one is looked up in the FLANN tree. In fact, $K$ nearest neighbors of each feature are found and the votes of the corresponding image in the database is increased by one, i.e. voting scheme. Thereafter, database images are ranked based on the collected votes. After finding the $K$ best matches, geometric consistency and SIFT features angle difference checks are performed to prune the matched features between the query image and best matches. Next, the best match is selected among $K$ best matches. Following this stage, they perform pose estimation (fine localization). In fact, they use phone sensors including accelerometer and magnetometer in order to find rotation angles. Since, this side information is not available in our problem, we only use their best match selection as a benchmark coarse localizer and refer to this work as \emph{Liang's method} in the experimental results.

The second benchmark is the method proposed in \cite{Torii}, which proposes a feature-based method for fine localization.
The proposed method is composed of two stages. In the first stage, it finds the best pair of images in the database. The best pair is defined as the pair of images that the linear combination of their image descriptors estimates the query descriptor with the least error.
Next, it linearly combines (i.e. interpolates) the locations of best matches with an interpolation factor in order to estimate the query location. 
It suggests that the interpolation factor be an affine function of the visual similarity between database and query images. The visual similarity is computed using the well-known bag of feature representation of database and query images. The offset and slope parameters of the linear relation between the interpolation factor and the visual similarity is learned on an independent database in the training phase. For image similarity computation, the bag of features representation is used in \cite{Torii}. In this paper,  we use the modified VLAD's representation \cite{allVLAD}, which has demonstrated supreme performance. We refer to this benchmark as \emph{Torii's method}.

In  \cite{Hamed3}, we proposed an OCR-aided localization system called OCRAPOSE. The system utilizes OCR to read the existing characters in the query image and provides a rough location estimate using character location-tagged floor plan. Afterwards, it performs OCR-aided stereo feature matching between the query and database images containing the same characters. The stereo matching results are used to estimate a homography matrix. The homography is used to find the world coordinates of the query features. Next, a PnP problem is solved for the query feature points and their world coordinates to obtain a fine location estimate. The main difference between the OCRAPOSE and the OCRAPOSE II proposed here is that teh new system does not require any world coordinates measurement, image database collection or point feature extraction/matching. Hence, it is more practical in terms of alleviating the requirement for huge image database collection  and avoiding point feature processing complexity. Furthermore, two novel projective Geometry-based formulas are proposed in OCRAPOSE II for depth and angle-of-view estimation.

All of the mentioned benchmarks require training images. Therefore, we have to collect a number of  training images in each scenario. It should be noted that our method does not require any image databases. The only side information required is the 3D locations of characters centroid in indoor scenarios as well as the width of the OCR boxes defined later. This information are required to perform query fine localization on the floor plan.




\section{The proposed system (OCRAPOSE II)}
\label{sec:system}


Fig. \ref{fig:diagram} depicts the structure of the proposed system. 
As seen, the proposed system detects the horizontal vanishing point in order to estimate the user's AOV with respect to the location of the centroid of the characters. Furthermore, it detects and recognizes the characters in the query image, which along with the floor plan provides the required information for rough localization. Next, using the estimated AOV and the width of the OCR box in the image and real world, depth of the query is estimated. Finally, a fine estimate of the user's location is computed using the estimated depth and AOV.
In the sequel, we explain the role of system blocks in details.

\begin{figure*}[!t]
\centering
\includegraphics[width=0.7\textwidth,height=0.22\textheight]{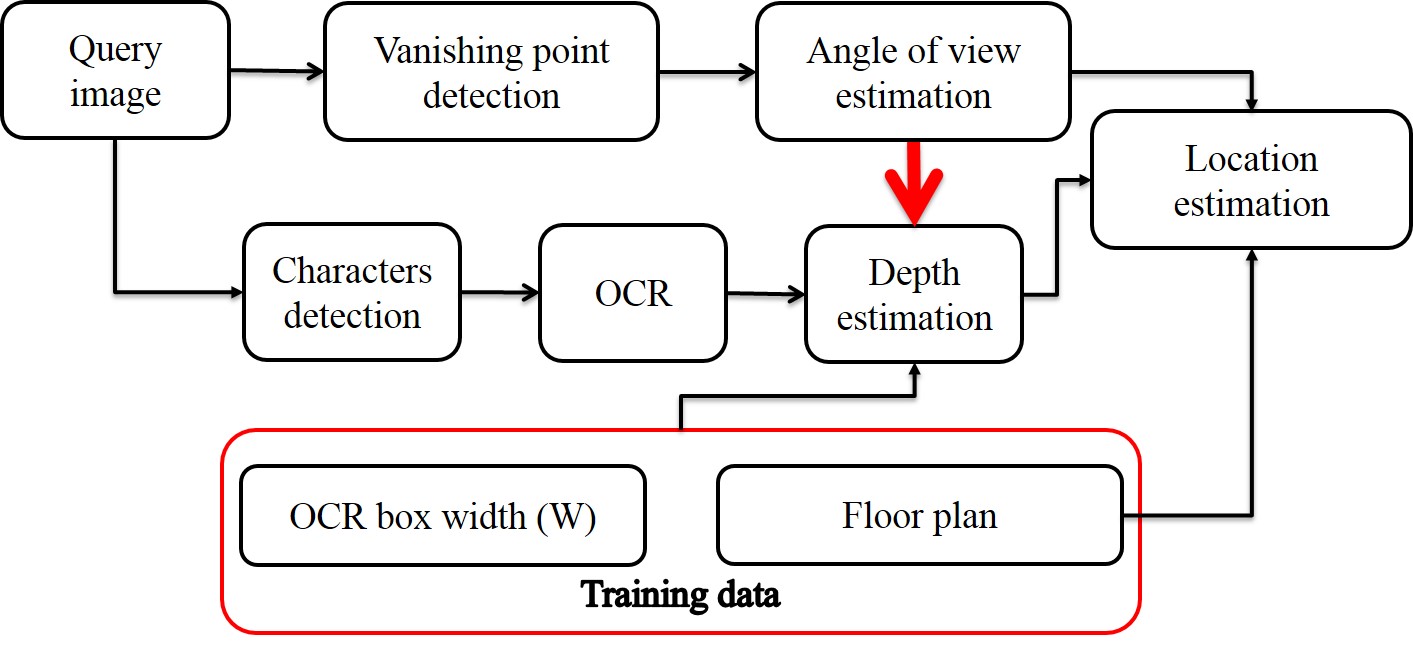}\\
\caption{Block diagram of the proposed localization system (OCRAPOSE II)}
\label{fig:diagram}
\end{figure*}


\subsection{Characters detection}
\label{sec:cdb}

As seen in Fig. \ref{fig:diagram}, the query image is input to the characters detection block.
The role of the characters detection block is to detect the regions of interest (RoIs) that contain the characters. False positive regions (non-character regions detected as character regions) are tolerable as long as  we are able to identify and remove them from the list of true positives.

Text detection in natural scene images is an extremely difficult problem \cite{text_TIP2, stanford} and might need parameters fine tuning in the scenario under investigation.
On the other hand, proposing a text detection algorithm is out of the scope of this paper. Hence, we use some of the state-of-the-art techniques to design our text detection block. It should be noted that utilizing a better text detection algorithm results in higher location recognition performance of the proposed method. Consequently, it results in larger performance gap between the proposed method and the feature-based ones.

In indoor scenarios, text is usually printed uniformly over a board (plate). It causes the text regions to appear with uniform colors (intensities) inside the query images. Such regions are good candidates to be detected as maximally stable extremal regions (MSER) \cite{MSER} as also suggested by \cite{stanford}. Hence, we extract MSER regions as potential characters.

In the next step, we perform Geometric filtering on the detected MSER regions. In works such as \cite{stanford}, this Geometric filtering is carried out after processing the MSER regions. We found this to be inefficient since it is computationally cheaper to remove the improbable regions first and avoid the required excess processing later. Therefore, we perform Geometric filtering first as follows and refine the MSER regions afterwards.
Define the set ($R$) of $n$ detected MSER regions as
\begin{equation}
R= \{ R_1, R_2, \cdots, R_n\}.
\end{equation}
Our Geometric filtering is composed of two steps. First, we remove large regions whose area (i.e. number of pixels) is more than $N$ times the median area of all regions.
Next, we remove regions whose orientation is more than $\epsilon$ degrees away from the vertical orientation. This is since each character is usually oriented vertically in a query camera image captured with zero pitch angle. 
  
Mathematically, a region $R_i$ passes through the Geometric filter if
  \begin{itemize}
\item $| R_i |< N \cdot \text{median} (|R|)$
\item $| \angle {R_i}  -90 \degree| < \epsilon$
\end{itemize}
   where $|R_i|=\text{\# pixels}(R_i)$, $|R|=\left(|R_1|, \cdots, |R_n|\right)^T$ and $\angle {R_i}=\text{Orientation}(R_i)$.  
In our experiments, we found appropriate values of $N$ and $\epsilon$ to be $10$ and $5$, respectively.


MSER region detector is known to be sensitive to image blur \cite{stanford}. Blurring causes the MSER detector to miss or detect distorted regions, which makes the OCR difficult. Hence, we enhance the MSER regions using \cite{stanford}. \cite{stanford} uses Canny edge detector to enhance the outline of detected extremal regions. In fact, it prunes the MSER regions along the detected gradient direction suggested by the Canny edge detector. 

We avoid further refinement or using other detection techniques proposed in the literature such as stroke width transform because of their complexity. The mentioned level of refinement turned out to be sufficient in our applications. We only perform simple preprocessing operations before inputting the image to the OCR engine as explained in Section \ref{sec:ocrblock}.


\subsection{OCR}
\label{sec:ocrblock}
The OCR block recognizes the existing characters in the detected regions. We perform some preprocessing prior to inputting the enhanced MSER regions to the OCR block. This preprocessing includes global binarization and removing relatively small/large regions. Next, we input the remaining regions to the OCR engine. 
In order to perform OCR, we use MATLAB OCR function, which is using Google Tesseract engine. The engine is based on convolutional neural network approach for character recognition. Tesseract is believed to be one of the most accurate OCR engines \cite{Tesseract}.


\begin{figure}
\center
\includegraphics[scale=0.5]{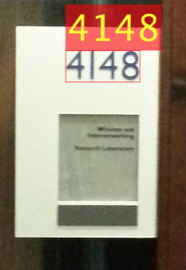}
\caption{Output of OCR}
\label{fig:OCR}
\end{figure}


\subsection{angle-of-view (AOV) estimation}
\label{sec:AOV}

We detect the line segments in the query image and utilize them in order to  estimate the horizontal vanishing point (VP). Next, we derive a novel formula for AOV estimation using the detected vanishing point.

\subsubsection{Vanishing point detection}
Vanishing points present in the image contain information about query camera rotation with respect to the seen objects \cite{VP2}. 
If the roll angle of the query camera is zero with respect to the horizon, horizontal (vertical) line segments will intersect in the horizontal (vertical) VP.  Moreover, the perpendicular (out of plane) VP can be obtained by the cross product of other two VPs.
Therefore, image line segments can help us find all vanishing points hence the complete query rotation matrix.

We perform line segments processing for VP detection. As we will explain in Section \ref{sec:AOV}, only horizontal VP is needed in our method. Hence, we detect it by solving a novel robust optimization problem partly inspired by \cite{berck13}.
In \cite{berck13}, line segments are detected using an edge-based method proposed in \cite{lsd}. We also use this method to extract line segments.
Afterwards, we estimate the horizontal vanishing point location in the image using detected line segments. In order to do this, we propose a novel robust optimization problem inspired by \cite{berck13} using RANSAC \cite{RANSAC}.

Assume the total number of  detected lines is $L$ and RANSAC selects $n$ line segments in each iteration. $n$ should be at least 2. In order to find the vanishing point, the proposed optimization problem finds the point in the image plane that has the minimum \emph{total weighted distances} from all $n$ lines considered in the current iteration (i.e. potential inliers). In fact,  in each iteration of RANSAC, we solve 
\begin{equation}
VP=\underset{\mathbf{x}}{\text{argmin}}\; {\sum_{i=1}^n {|L_i| \; \left(dist(\mathbf{x},L_i)\right)^2}  }
\label{eq:vp_opt}
\end{equation}
where $dist(\mathbf{x},L_i)$ represents the distance of image point $\mathbf{x}$ from the $i^{th}$ line (i.e. $L_i$). Furthermore, $|L_i|$ stands for the length of $L_i$ in terms of number of pixels.
In fact, we minimize the weighted sum of horizontal VP distances from the detected lines. Hence, the solution  is essentially the MAP estimate of the VP in which the priors are assigned proportional to the line lengths. Moreover, the distance noise is assumed Gaussian. Under these conditions, minimizing the proposed cost function maximizes the A-posteriori probability.
The proposed minimization problem has a closed form solution that follows.

\begin{proposition}
\label{prep:VP}
Assume the $i^{th}$ almost horizontal line ($L_i$) detected in the image is defined as follows
\begin{equation}
L_i : a_i x+b_i y +c=0 \quad \text{for} \; i=1,\cdots, n.
\end{equation}
The image coordinates of the horizontal VP can be estimated as
\begin{eqnarray}
\label{eq:VP_est}
VP_x=& \frac{BF-CE}{AE-BD} \nonumber\\
\quad\\
VP_y=& \frac{CD-AF}{AE-BD}
\end{eqnarray}
where
\begin{eqnarray}
 A =&\sum_{i=1}^n{\frac{1}{\gamma_i}|L_i| a_i^2} \nonumber\\
 B=& \sum_{i=1}^n{\frac{1}{\gamma_i}|L_i| a_i b_i} \nonumber\\
C=&\sum_{i=1}^n{\frac{1}{\gamma_i}|L_i| a_i c_i}\nonumber\\
D=&B \nonumber\\
E=&\sum_{i=1}^n{\frac{1}{\gamma_i}|L_i| b_i^2} \nonumber\\
F=&\sum_{i=1}^n{\frac{1}{\gamma_i}|L_i| b_i c_i}
\end{eqnarray}
and $\gamma_i=a_i^2+b_i^2$.
\end{proposition}

\begin{proof}
The unconstrained optimization problem defined in (\ref{eq:vp_opt}) should be solved in order to find the VP. In order to perform this, we should set the derivative of the cost function
\begin{equation}
\label{eq:VP_cost}
Cost=\sum_{i=1}^n{|L_i| (dist(x,L_i))^2}
\end{equation}
with respect to the VP coordinates to zero. $dist(\mathbf{x},L_i)$ is the distance of the VP to the $i^{th}$ line
\begin{equation}
dist(\mathbf{x},L_i)=\frac{| a_i x+ b_i y + c|}{\sqrt{a_i^2 + b_i^2}}
\end{equation}
Setting the derivatives to zero as
\begin{equation}
\frac{ \partial{(Cost)}}{\partial{x}}=0, \; \;
\frac{ \partial{(Cost)}}{\partial{y}}=0
\end{equation}
results in (\ref{eq:VP_est}).
\end{proof}

Fig. \ref{fig:VP} shows the result of applying the proposed method on a sample query image taken inside a university building.
\begin{figure}[!t]
\centering
\includegraphics[width=0.5\textwidth,height=0.25\textheight
]{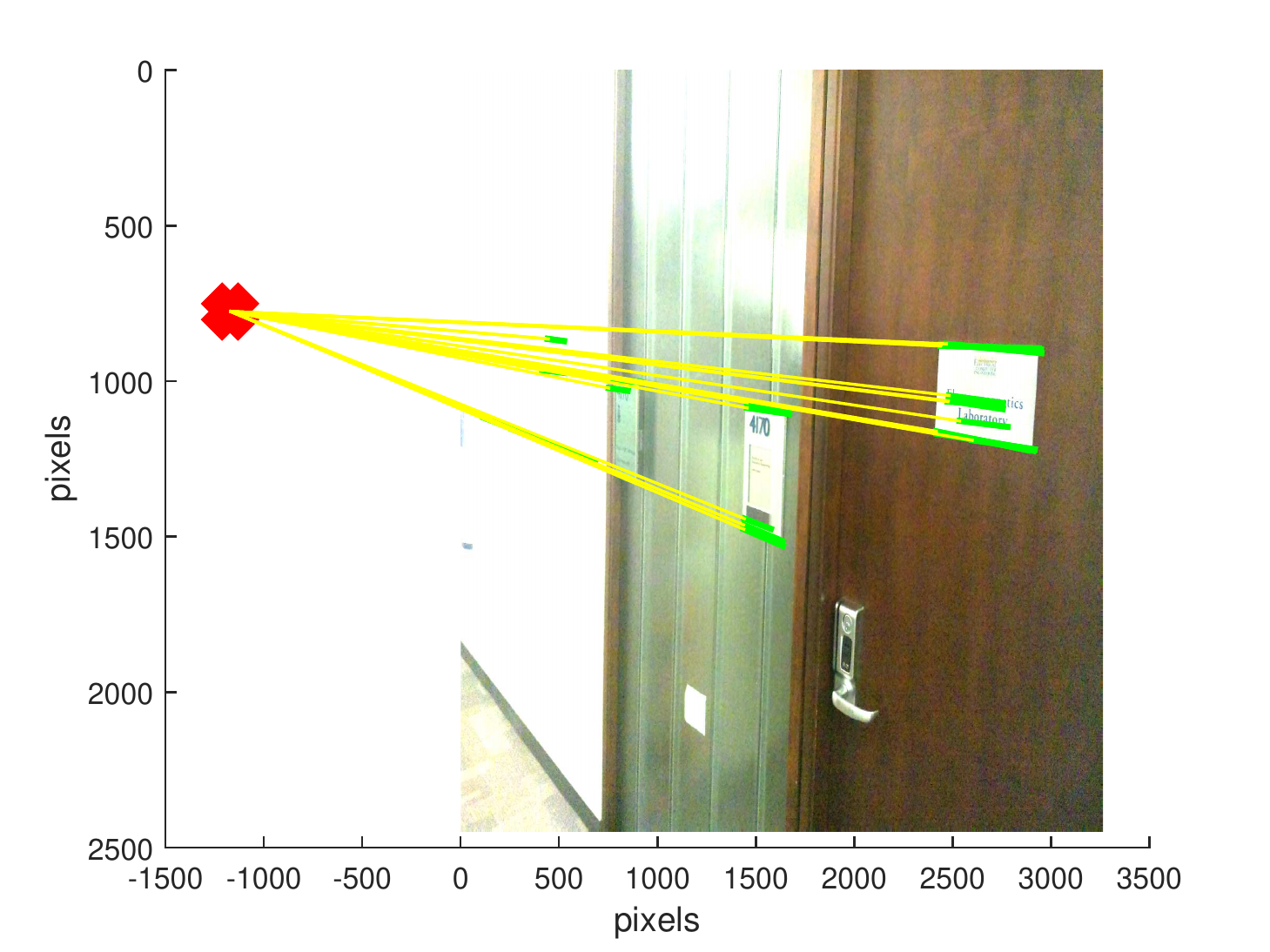}
\caption{Detected line segments (green) and horizontal VP (red cross), extrapolated semi-lines from the line segments are shown in yellow}
\label{fig:VP}
\end{figure}
As stated, in order to derive practical formulas for AOV and depth estimation, we make some assumptions about the query rotation angles with respect to the characters centroid (i.e. reference frame origin). First, we assume that the roll angle is zero. This assumption can be easily relaxed by some image preprocessing. As stated in \cite{minimal}, if the vertical vanishing point (i.e. $VVP=(x_{ver} \; y_{ver})^T$) is known, the roll angle can be estimated as

\begin{equation}
roll=\arctan(\frac{x_{ver}}{y_{ver}})
\end{equation}

Vertical vanishing point can also be estimated using a method similar to the one proposed for the horizontal VP.
Once the roll angle is estimated, the query image can be rotated accordingly in order to remove the non-zero roll effect. For human user localization, we need to perform this pre-processing to compensate for nonzero roll.
 
We assume the tilt angle is zero. This assumption is not exactly met in practice but the practical values are fairly close to zero. The difference comes from the height difference between the characters centroid and camera center. In Section \ref{sec:nonzerotilt}, we study the range of (non-zero) tilt angles in practice and demonstrate their negligible effect on the accuracy of AOV and depth formulas. 

The last assumption is about the pan angle. For indoor human user localization, we assume that the pan angle is equal to the AOV as depicted in Fig. \ref{fig:scenario}. This is a reasonable assumption since humans tend to orient their phone towards the point of interest. Technically speaking, they effectively align the normal vector of the image plane with the line connecting the target (characters centroid) to the camera center. The motivation is to place the location of the OCR detected characters in the center of attention, i.e. the middle of the image. It can be used as a test to see whether the pan angle conditions are met. 

Since tilt and roll are assumed to be zero, the only unknown angle is pan. As mentioned in \cite{Liang2013}, in indoor images, most of the visible objects are located on a single wall. In conclusion, lines seen (almost) horizontal/vertical in the query image are actually horizontal/vertical in the real world. 

Consider the scenario depicted in Fig. \ref{fig:scenario}. As seen, the pan angle is equal to the AOV. The location of horizontal VP is a function of the pan angle. Hence, knowing the horizontal VP, we can estimate the AOV. In the following section, we propose  a novel formula that estimates the AOV in the scenario depicted in Fig. \ref{fig:scenario}. 

\begin{figure}[t!]
\center
\includegraphics[height=0.3\textheight, width=.45\textwidth]{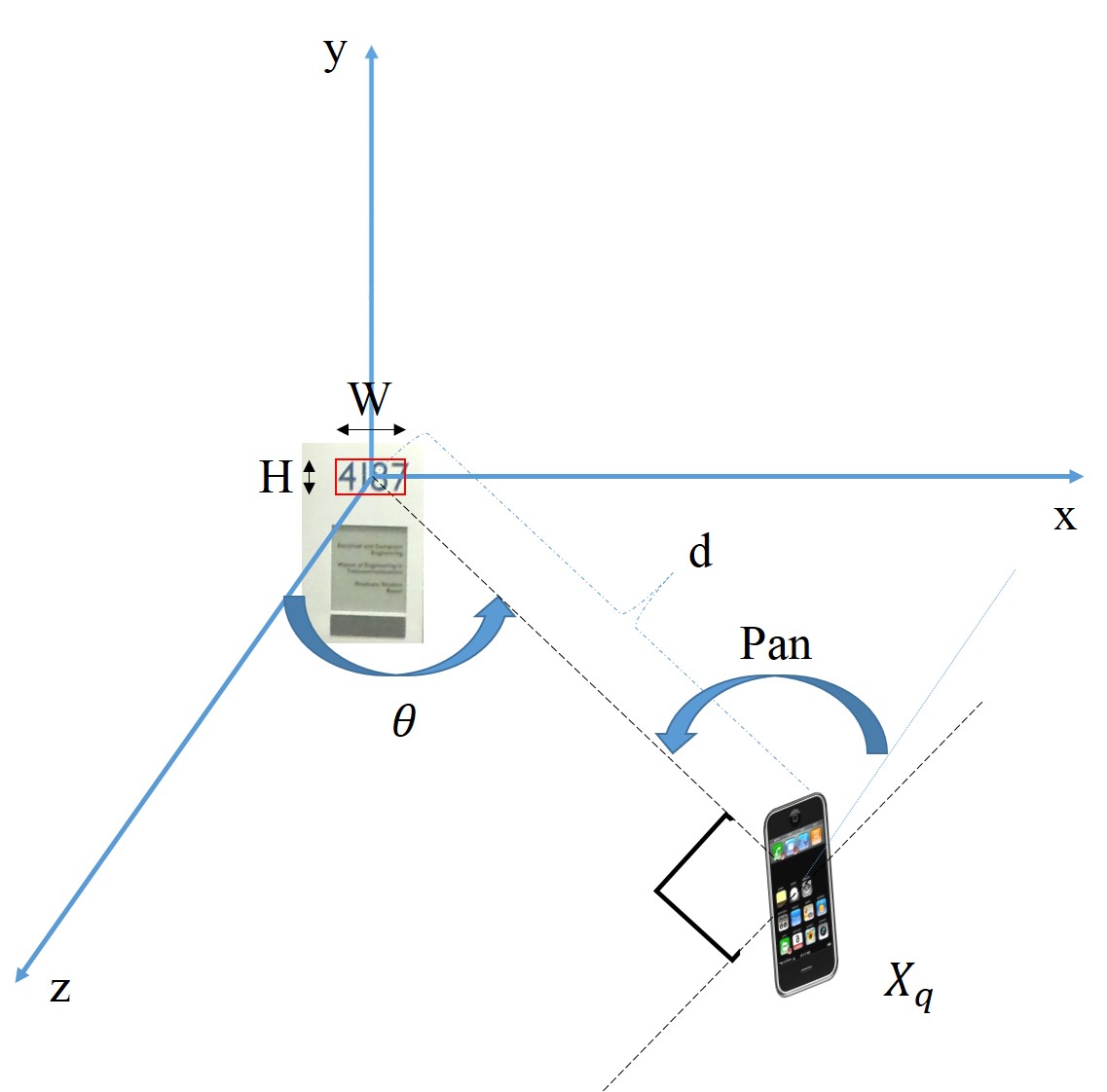}
\caption{The indoor scenario considered in this paper}
\label{fig:scenario}
\end{figure}

\subsubsection{Estimating AOV from the horizontal vanishing point}

In this section, we propose a novel formula for estimating the AOV in the practical scenario depicted in Fig. \ref{fig:scenario} using the estimated horizontal VP. We only use the OCR box corners in the derivation of the formula and it is not required to detect corners in the actual query phase in practice.
The OCR box is a fictitious box surrounding the characters as depicted in Fig. \ref{fig:OCRbox2}. 

The approximate formula that will be derived for AOV estimation is independent of the dimensions of the OCR box. It is only the depth estimation formula that requires the actual width of the box.
Hence, what we need to do is to only measure the width ($W$) of the OCR (characters) box in the environment. Since, this value is the same with high probability for the text/number plates in a single building, we only need to perform the measurement once. If the character boxes are different in width size in an environment, it is possible to tabulate the sizes and use them once the OCR engine has identified the characters. 

\begin{figure}
\center
\includegraphics[width=0.3\textwidth,height=0.2\textheight]{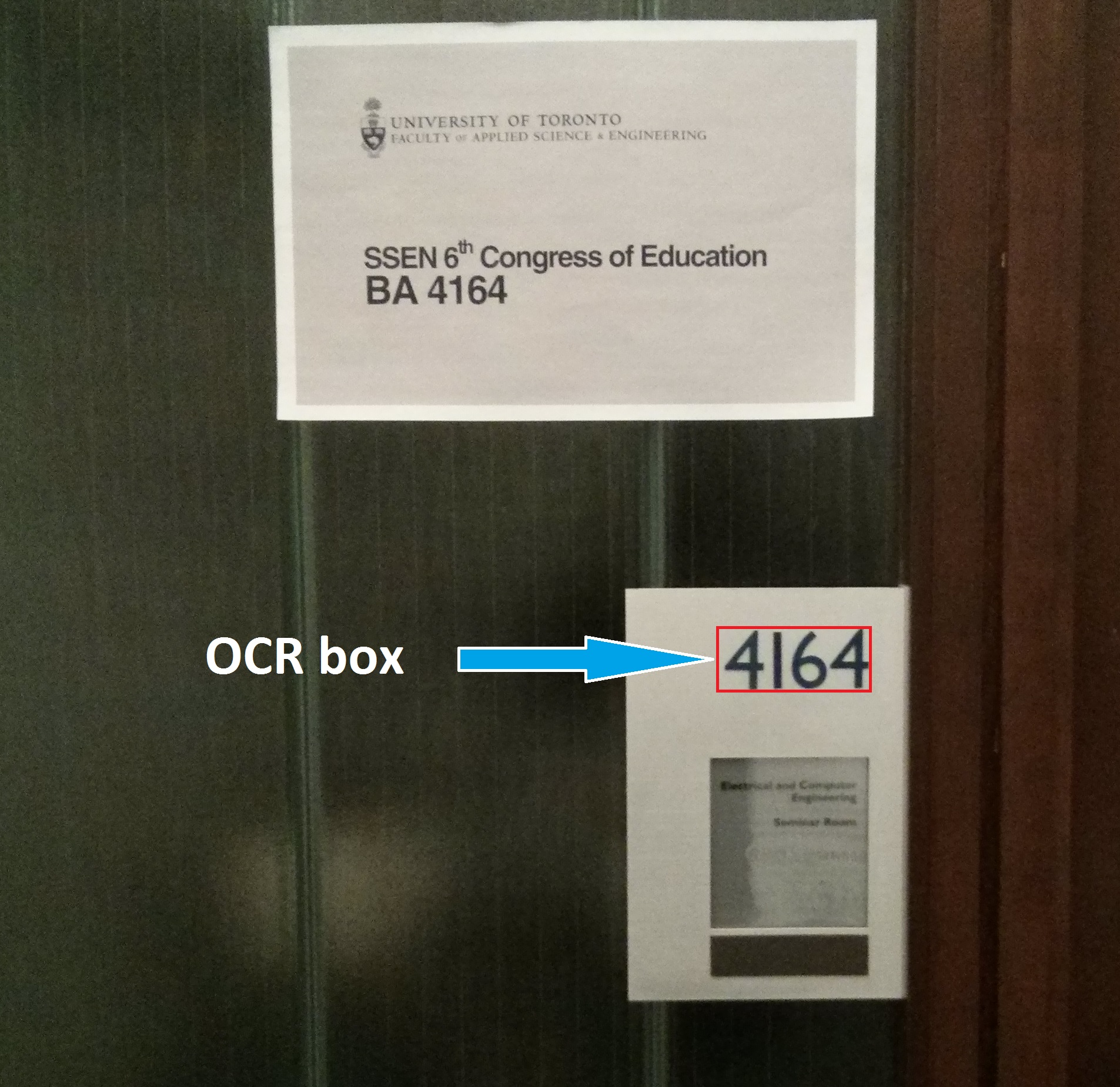}
\caption{A sample OCR box detected by the Tesseract engine}
\label{fig:OCRbox2}
\end{figure}

In order to derive the AOV formula, we consider an OCR box of arbitrary dimensions. We analyze the projection of this box in the query image. From the projected coordinates, a formula is derived for obtaining the horizontal VP. By manipulating  this formula, a closed-form formula is derived for AOV estimation in terms of horizontal VP location. Proposition \ref{prep:AOV} summarizes the results.

\begin{proposition}
\label{prep:AOV}
\textbf{\ignore{Agle of view estimation formula}} If the actual width of the OCR box, depth and AOV of the query are $W$, $d$ and $\theta$, respectively, the $x$ coordinate ($x_{hor}$) of the horizontal VP can be obtained as
\begin{equation}
x_{hor}=- \frac{\cos{\theta} (4d^2-W^2 \sin^2{\theta} + 2d W \sin{\theta} ) } { 4d^2\sin{\theta} - W^2 \sin^3{\theta}}
\label{eq:x_hor}
\end{equation}
\end{proposition} 

\begin{proof}
Without loss of generality, assume the OCR box is located at the origin of 3D reference frame as depicted in Fig. \ref{fig:scenario}. Furthermore, according to the scenario depicted in Fig. \ref{fig:scenario}, the user, i.e. the query camera, is located at 
\begin{equation}
\mathbf{X}_q=\begin{pmatrix}d\sin{\theta} \\\\ 0\\\\ d\cos{\theta} \\\\ 1 \end{pmatrix}
\end{equation}
where $d$ and $\theta$ are depth and AOV of the query with respect to the characters centroid. 
We use normalized camera matrix (i.e. $\mathbf{P}'=\mathbf{K}^{-1} \mathbf{P}$) in our analysis to make it independent from the camera calibration matrix.
Using our assumptions, the normalized camera matrix can be written as
\begin{equation}
\mathbf{P}'=\mathbf{R} \left( \mathbf{I} \; | -\mathbf{X}_q\right)
\end{equation}
where $\mathbf{R}$ and $\mathbf{I}$ are the query camera rotation matrix and the $3 \times 3$ identity matrix, respectively.
$\mathbf{R}$ can be written as
\begin{equation}
\mathbf{R}=\mathbf{R}_{roll} \; \mathbf{R}_{tilt} \; \mathbf{R}_{pan}
\end{equation}
where
\begin{eqnarray}
\mathbf{R}_{roll}=&\begin{pmatrix} 1 & 0 & 0\\ 0 & 1 & 0\\ 0 & 0 & 1   \end{pmatrix}\\
\mathbf{R}_{tilt}=&   \begin{pmatrix} 1 & 0 & 0\\ 0 & -1 & 0\\ 0 & 0 & -1  \end{pmatrix}\\
\mathbf{R}_{pan}=& \begin{pmatrix} \cos{\theta} & 0 & -\sin{\theta}\\ 0 & 1 & 0\\ \sin{\theta} & 0 & \cos{\theta}   \end{pmatrix}
\end{eqnarray}
As seen, the roll angle has been set to $180 \degree$ to account for pointing the camera towards the reference frame origin, which is an inwards direction.

As depicted in Fig. \ref{fig:OCRcloseshot}, $\mathbf{X}_1, \cdots, \mathbf{X}_4$ represent the 3D coordinates of the box corners. Define $\mathbf{X}_{box}=(\mathbf{X}_1, \cdots, \mathbf{X}_4)$. 
One can verify that $W=\| \mathbf{X}_2 - \mathbf{X}_1\|=\| \mathbf{X}_3-\mathbf{X}_4 \|$ and $H=\| \mathbf{X}_4 - \mathbf{X}_1\|=\| \mathbf{X}_3-\mathbf{X}_2 \|$.
$\mathbf{X}_{box}$ can be represented in terms of $W$ and $H$ as
\begin{equation}
\mathbf{X}_{box}=\begin{pmatrix}
-\frac{W}{2}  & \frac{H}{2}    &  0 &  1\\\\
\frac{W}{2}   &   \frac{H}{2}  &  0 &  1\\\\
\frac{W}{2}   & -\frac{H}{2}   &  0 &  1\\\\
-\frac{W}{2}  & -\frac{H}{2}   &  0 &  1
\end{pmatrix}
\end{equation}

\begin{figure}
\center
\includegraphics[scale=0.45]{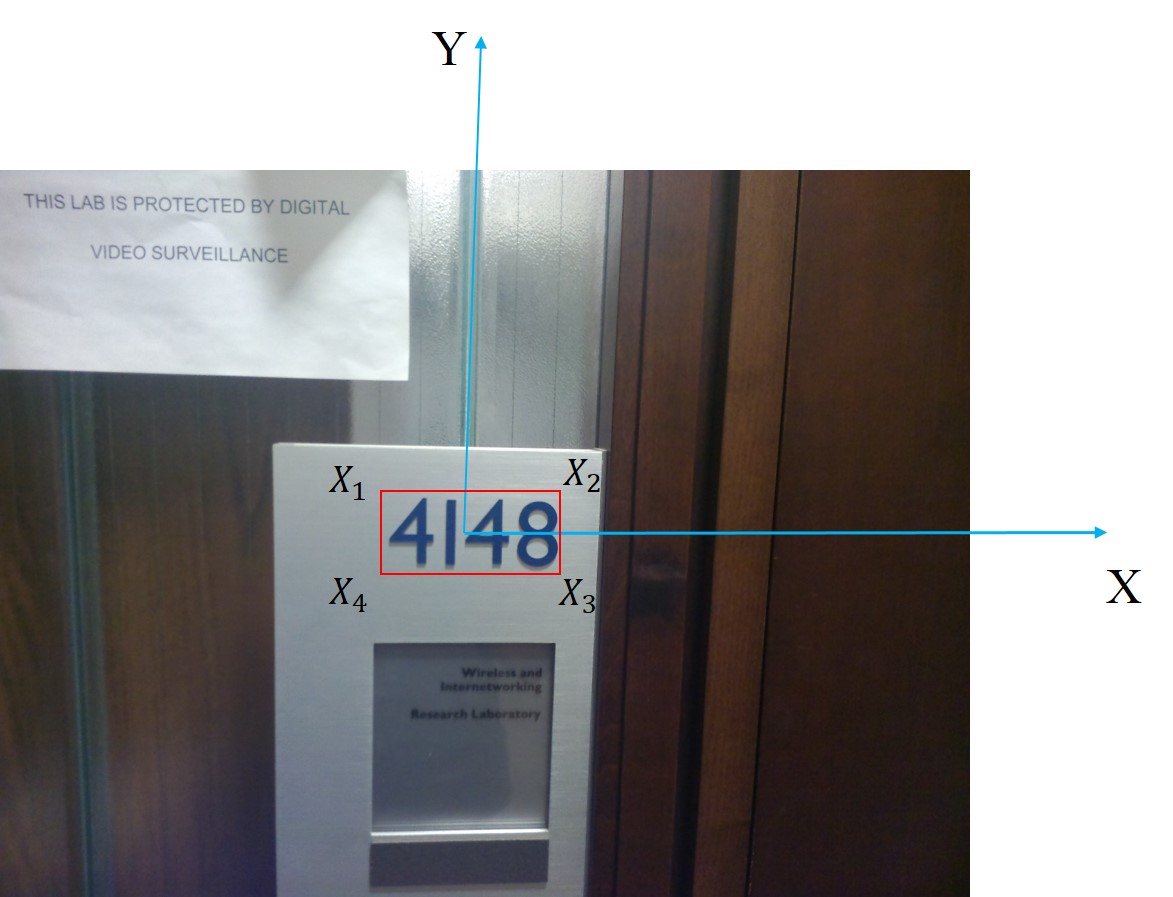}
\caption{Close shot of an OCR box}
\label{fig:OCRcloseshot}
\end{figure}

The image coordinates of corner points, i.e. $\mathbf{x}_1, \cdots, \mathbf{x}_4$, can be obtained as
\begin{equation}
\mathbf{x}_{box}=\mathbf{P}' \mathbf{X}_{box}
\label{eq:xbox}
\end{equation}

Equation (\ref{eq:xbox}) gives an expression for $\mathbf{x}_{box}$ as a function of  $W, H$, $\theta$ and $d$.  We find the horizontal VP and the width of OCR box using these coordinates.

Here, the goal is to find the horizontal VP in terms of $\theta$, W and H. This can be done by intersecting two `almost' horizontal lines connecting $\mathbf{x}_1$ to $\mathbf{x}_2$ and $\mathbf{x}_4$ to $\mathbf{x}_3$, respectively. 
Having done this, the horizontal VP is obtained as
\begin{equation}
VP_{hor}=\begin{pmatrix}
 x_{hor}\\\\
0                                                 
\end{pmatrix}
\end{equation}
where after simplifying results in (\ref{eq:x_hor}).
\end{proof}

As seen, $x_{hor}$ is not a function of H.  The proposed formula can be approximated to derive a simple formula for AOV as follows.
\begin{corollary}
The AOV can be estimated by approximating (\ref{eq:x_hor}) as
\begin{equation}
\theta \approx -\frac{1}{x_{hor}}
\label{eq:golden1}
\end{equation}
\end{corollary}

 \begin{proof}
Assume
\begin{eqnarray}
& \sin{\theta} \approx \theta \\
& \cos{\theta} \approx 1
\end{eqnarray}
Applying these assumptions to (\ref{eq:x_hor}), we get 
 \begin{equation}
 x_{hor} \approx -\frac{1}{\theta}
 \label{eq:xhorapprox}
\end{equation}
Solving for $\theta$ results in (\ref{eq:golden1}).
\end{proof}

As seen, the AOV formula is independent of depth as desired. It is also independent of W. So, using this approximation, there is no need to measure any dimensions in the training phase for the AOV estimation .

Fig. \ref{fig:xhor} depicts the true and approximated values of $x_{hor}$. As seen, the approximation is close to the true value for almost the entire range of $[-45\degree, 45\degree]$.
\begin{figure}[!t]
\centering
\includegraphics[scale=0.6]{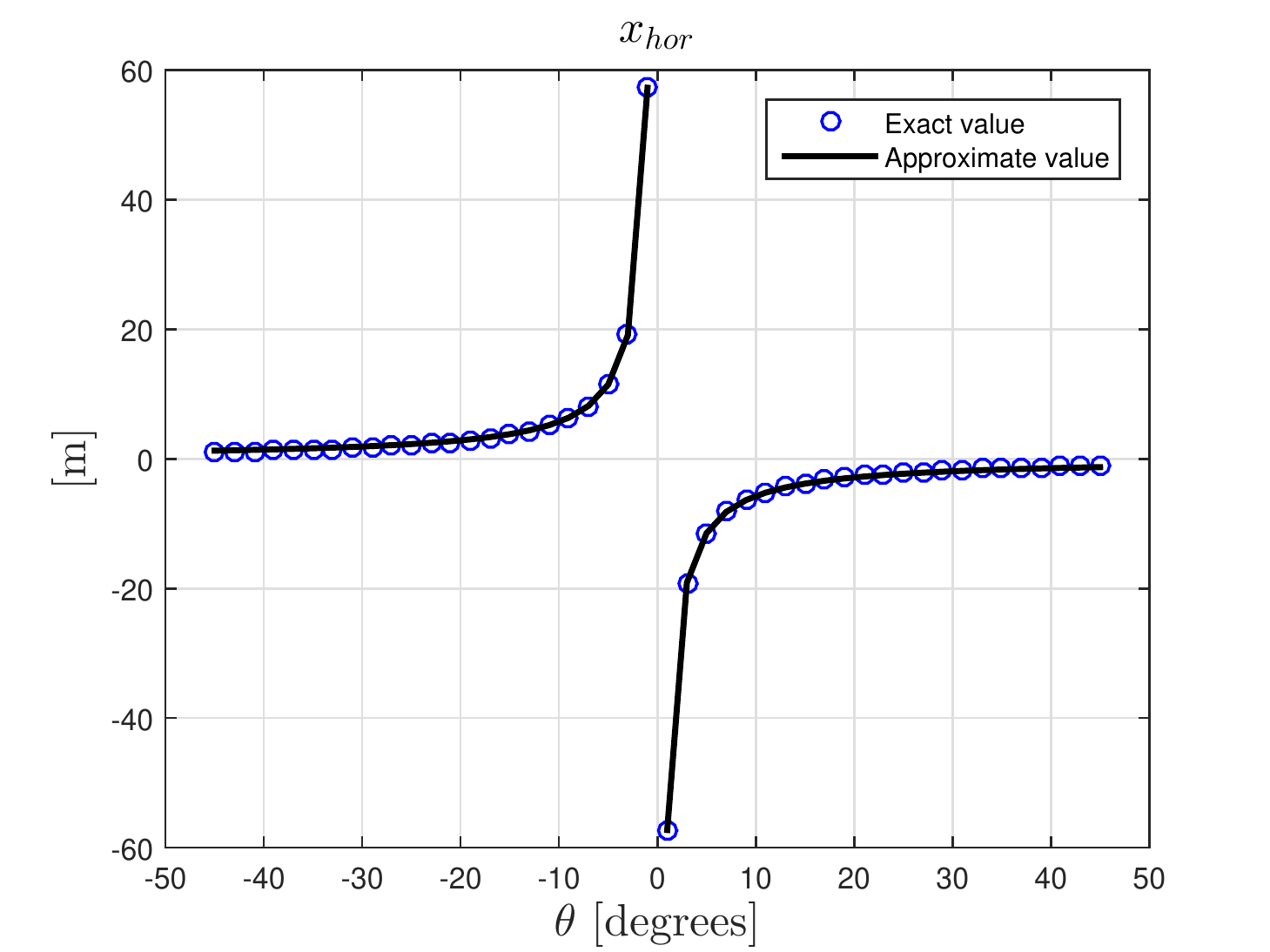}\\
\caption{Exact and approximate value of $x_{hor}$ [meters] w.r.t. $\theta$ [degrees]}
\label{fig:xhor}
\end{figure}
Fig \ref{fig:error_xhor} depicts the error in $\theta$  estimation for the mentioned range.
As seen, the closer the $\theta$ to the boundaries of the interval, the worse the approximation. The RMS value of the error for the entire range is $4.68 \degree$. 
\begin{figure}[!t]
\centering
\includegraphics[scale=0.6]{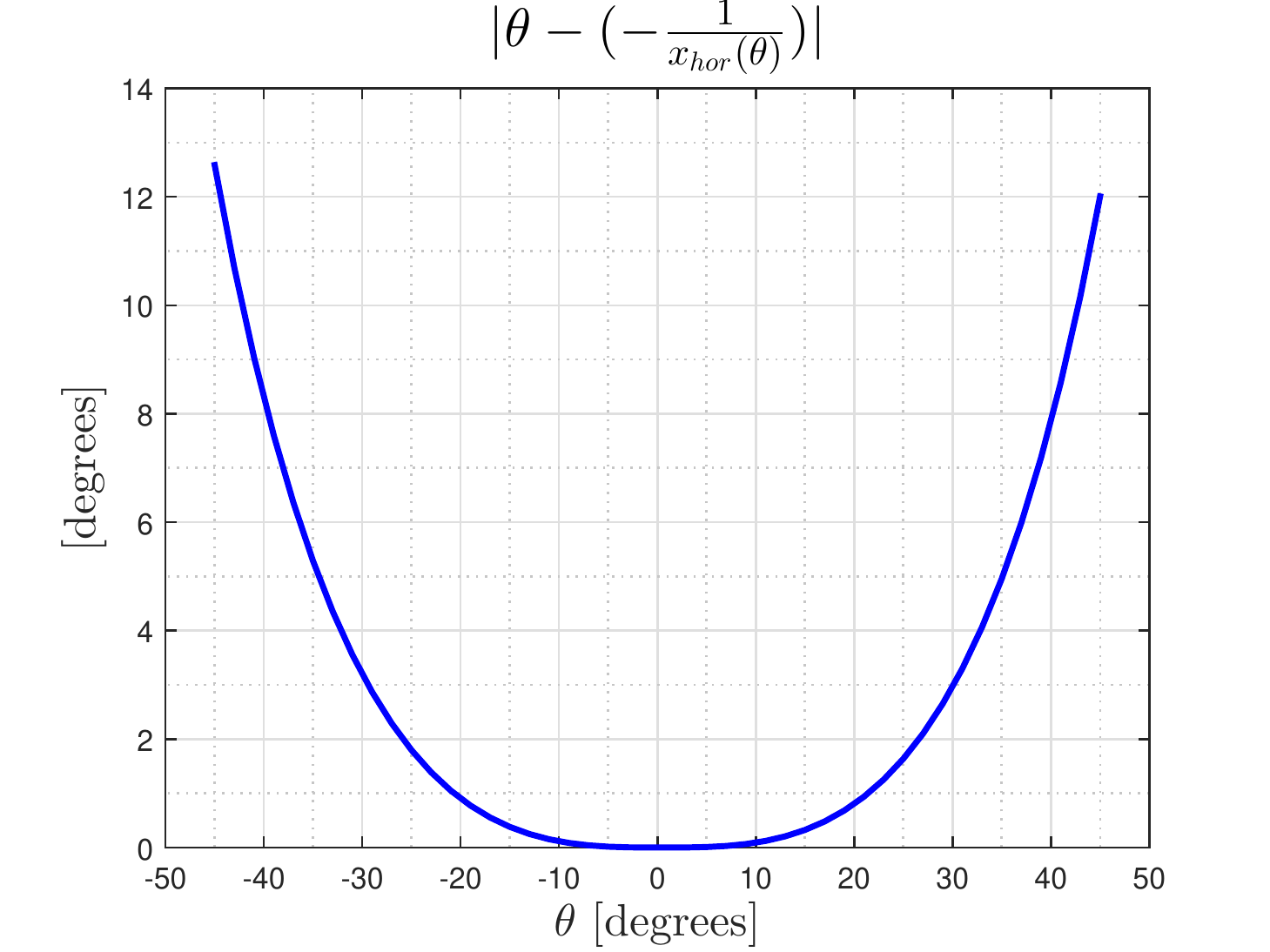}\\
\caption{Error in $\theta$ estimation [degrees] computed at different AOVs using a SAMSUNG Galaxy S5 cellphone camera calibration parameters}
\label{fig:error_xhor}
\end{figure}

\ignore{
One important measure of formulas robustness is sensitivity. The sensitivity of function $f$ with respect to variable $a$ is defined as
\begin{equation}
S_{fa}= \frac{\partial f (a)} {\partial a}
\end{equation}
It is an important measure since it shows how the function value changes if small error happens in the estimation of parameter as
\begin{equation}
\Delta f \approx S_{fa} \Delta a
\end{equation}
Hence, it can be interpreted as a measure of error analysis.}
 We investigate the sensitivity as a measure of robustness for the derived formulas. For, we take the derivative of the formula derived for the quantity of interest with respect to the parameters in the problem.
Fig. \ref{fig:sensitivity_xhor} depicts the sensitivity of the estimated $\theta$ over a range of $x_{hor}$ values for the SAMSUNG Galaxy S5 cellphone. For instance, for the $x_{hor}$ of 1000 pixels, which is corresponding to a $\theta$ of about $23 \degree$, the sensitivity is less than 0.2 degrees/pixel.
\begin{figure}[!t]
\centering
\includegraphics[scale=0.6]{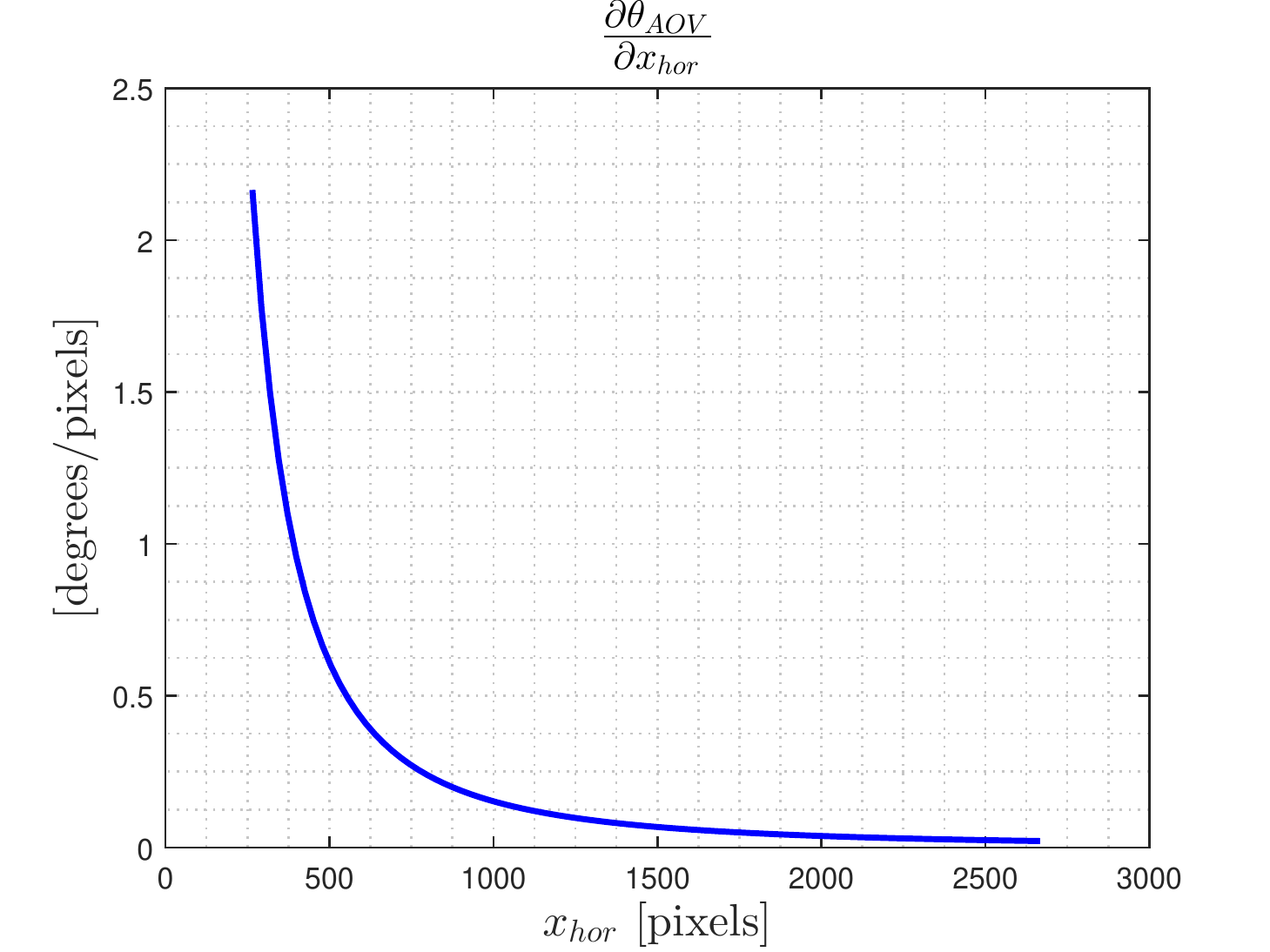}\\
\caption{Sensitivity of $\theta$ w.r.t. $x_{hor}$ [pixels] computed using a SAMSUNG Galaxy S5 cellphone camera calibration parameters}
\label{fig:sensitivity_xhor}
\end{figure}


Having recognized the existing characters, a rough estimate of the user's location can be obtained using the floor plan. Moreover, the user's AOV can be estimated from (\ref{eq:golden1}). The only missing information about the user's fine location is its distance (depth) to the characters centroid. Hence, we input the OCR output image to the depth estimation block. Notice that the existence of the OCR block is required since it provides a rough location estimate and side information required for depth estimation. 

\subsection{Depth estimation}
\label{sec:depth}
We propose a new practical formula for depth estimation, which utilizes the actual and image width of the OCR box.

\subsubsection{Depth estimation using width of OCR box}
\label{sec:depthw.r.t.ocrwidth}
It is obvious that as user gets further from the characters center, the width of OCR box in the query image gets smaller. Here, we formulate this relation.

\begin{proposition}
\label{prep:depth}
\textbf{\ignore{Depth estimation formula}} If the user's angle-of-view is $\theta$ and actual and image width of the OCR box are $W$ and $w$, respectively, user's depth can be calculated as
\begin{equation}
d=\frac{1}{2} \cos{\theta} \frac{W}{w} (1+\sqrt{1+w^2 \tan^2{\theta}}) \nonumber
\end{equation}
where $\theta$ is the query AOV.
\end{proposition}

\begin{proof}
Consider the OCR box in Fig. \ref{fig:OCRbox2}. For AOV estimation, we studied top and bottom horizontal line segments of this box. Here, we calculate the box width ($w$) and show that it is a function of depth, AOV and $W$. 
Similar to the angle-of-view formula derivation, we only use the corners coordinates in order to derive the depth formula. In fact, it is not required to detect the corners in the actual query phase.

Consider matrix $\mathbf{x}_{box}$ that contains the image coordinates of OCR box corners. Due to projective distortion, the image of the OCR box is a quadrilateral in general. 
Considering the fact that vertical lines remain vertical if there is neither in-plane rotation (or compensated for) nor tilt angle, $w$ can be calculated as
\begin{equation}
w=x_2-x_1=x_3-x_4
\label{eq:w}
\end{equation}
Similar to the $\theta$ formula derivation, we know the complete query camera matrix in terms of $d$, $\theta$ and 3D coordinates of the OCR box corners. Hence, using the parameterized camera matrix to find the OCR box corner location in the image, that is $\mathbf{x}_i$s, and using (\ref{eq:w}) one can get
\begin{equation}
w=\frac{4dW\cos{\theta}}{4d^2 - W^2 \sin^2{\theta}}
\label{eq:wexact}
\end{equation}
Exact depth formula can be derived by solving for $d$ in (\ref{eq:wexact}). It will result in a second-order equation as
\begin{equation}
(4w) \; d^2-(4W\cos{\theta}) \; d-wW^2\sin^2{\theta}=0
\end{equation}
The solutions are
\begin{equation}
d_{1 \& 2}=\frac{1}{2} \cos{\theta} \; \frac{W}{w} (1 \pm \sqrt{1+ w^2\tan^2{\theta}})
\end{equation}
where $d_2$ is the width in the image plane reflection with respect to the camera center. Hence, only $d_1$ is acceptable and we have
\begin{equation}
d= d_1 = \frac{1}{2} \cos{\theta} \; \frac{W}{w} (1 + \sqrt{1+ w^2\tan^2{\theta}})
\label{eq:dexact}
\end{equation}
\end{proof}

The depth formula is in concord with our intuition that smaller $w$ corresponds to greater distance ($d$) from the characters.
Fig. \ref{fig:sensitivity_theta} shows the sensitivity of the depth estimate with respect to the estimated AOV, i.e. $\frac{\partial d}{\partial \theta}$. As seen, in a distance of $7m$ from the characters center and at an AOV of 60 degrees, which is one of the worst cases for localization accuracy, each 1 degree of error in AOV will result in about $0.1 m$ of error in depth estimate.

Fig. \ref{fig:sensitivity_w} depicts the sensitivity of depth estimate w.r.t. $w$ (in pixels) for the SAMSUNG Galaxy S5. For, we used the calibration matrix of the phone camera to obtain the sensitivity in meters/pixel. As seen, when $w \approx 50 \; \text{pixels}$, 1 pixel of error in $w$ results in about $10~ cm$ of error in depth estimation.
 \begin{figure}
 \center
 \includegraphics[scale=0.6]{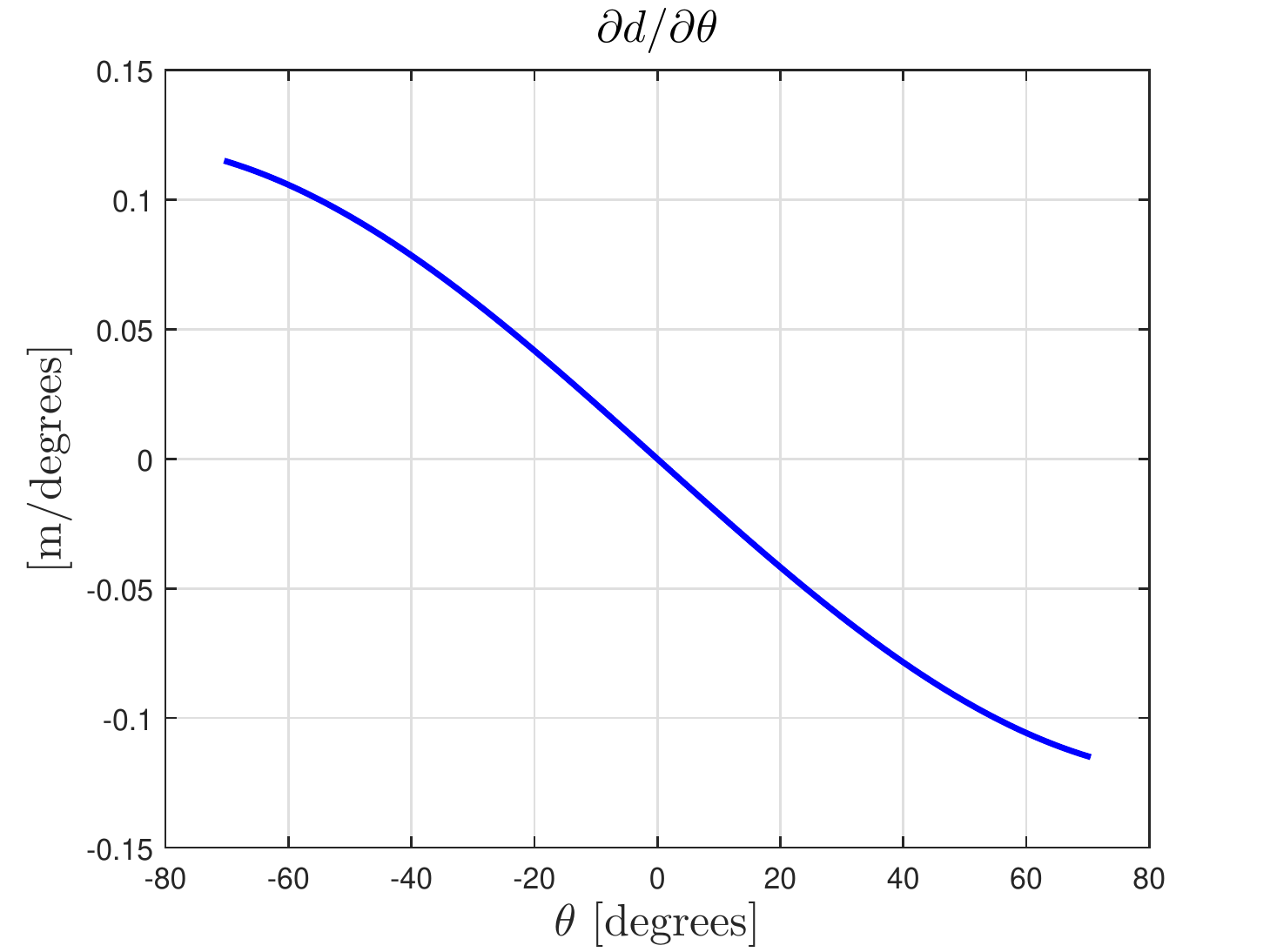}
 \caption{Depth sensitivity versus $\theta$ for $d=7~ m$, $W=0.1~m$ and $w=\frac{W}{d}$}
\label{fig:sensitivity_theta}
 \end{figure}
  \begin{figure}
 \center
 \includegraphics[scale=0.6]{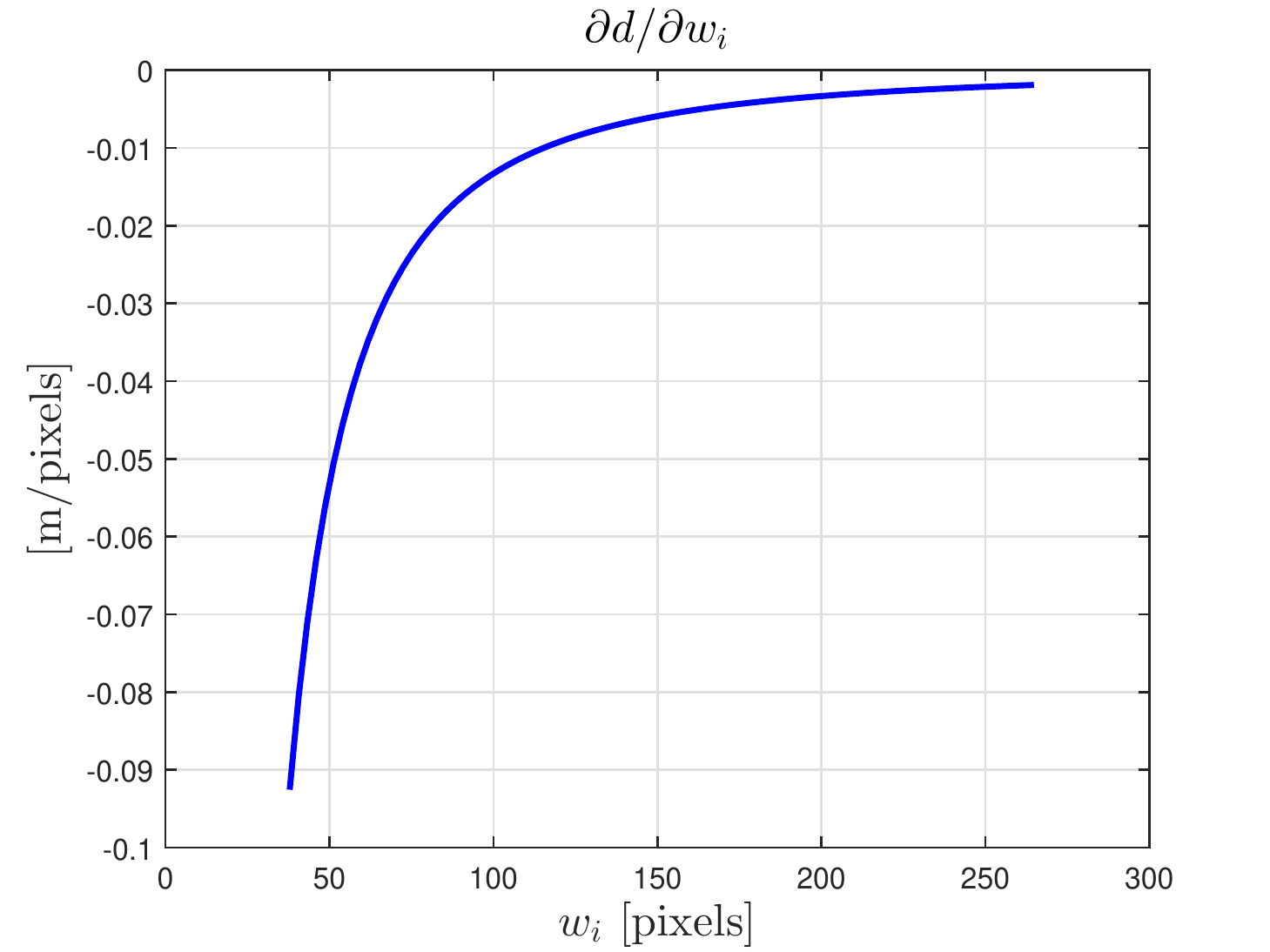}
\caption{Depth sensitivity versus the image width of OCR box for $d=7~m$, $W=0.1~m$ and $w=\frac{W}{d}$}
 \label{fig:sensitivity_w}
 \end{figure}
Finally, the sensitivity of the depth formula is linear with respect to the measured $W$.

\subsection{Location estimation}
\label{sec:locest}
Once depth and AOV are computed, one can obtain the $X$ and $Z$ coordinates of the user location for the scenario depicted in Fig. \ref{fig:scenario} as
\begin{eqnarray}
& X=d \sin{\theta} \nonumber\\
& Z=d \cos{\theta} 
\end{eqnarray}
The $y$ coordinate is not of interest for 2D localization in indoor applications, whereas it is essential for other applications such as augmented reality (AR).


\section{Non-zero tilt angle effect on AOV and depth estimation error}
\label{sec:nonzerotilt}

In the derivation of $x_{hor}$ and  $w$ formulas, we assumed tilt angle ($\phi$) is zero. In order to study the effect of non-zero tilt angle in practice, we re-derive the formulas for non-zero tilt angle as

\begin{equation}
x_{hor} =  \frac{\cos{\theta}}{2 d \cos{\phi} \sin{\theta}} (\frac{s_1+s_2}{s_3+s_4} ) 
\label{eq:new_xhor}
\end{equation}
and
\begin{equation}
w =  \begin{cases} x_2-x_1, & \phi >0 \\ x_3-x_4, & \phi<0 \end{cases}  =| \frac{s_5}{s_6+s_7} |
\end{equation}
where
\begin{eqnarray}
s_1=& 8 d^3 \cos^2{\phi}+ 2 d H^2 \sin^2{\phi}- 2 d W^2 \cos^2{\phi} \sin^2{\theta} \nonumber\\
s_2=& 4 d^2 H \sin{2 \phi} + 4 d^2 W \cos^2{\phi} \sin{\theta} - H^2 W \sin{\theta} \sin^2{\phi} \nonumber\\
s_3=& 4 d^2 \cos^2{\phi} + H^2 \sin^2{\phi} \nonumber\\
s_4=& - W^2 \cos^2{\phi} \sin^2{\theta} +2 d H \sin{2 \phi} \nonumber\\
s_5=& 2 W \cos{\theta} (2 d \cos{\phi} - H \sin{\phi}) \nonumber\\
s_6=& 4 d^2 \cos^2{\phi} + H^2 \sin^2{\phi} \nonumber\\
s_7= & - W^2 \cos^2{\phi} \sin^2{\theta} - 2 d H \sin{2 \phi} 
\end{eqnarray}
It should be considered in the derivation of $w$ formula that the rectangular box provided by the OCR engine always frames the entire area of the characters. Hence, the box width is equal to the maximum horizontal distance among $x_i$s based on the sign of $\phi$ as stated in the formula.
 
Equation (\ref{eq:new_xhor}) is not obviously the same as (\ref{eq:x_hor}) hence cannot be approximated in the same way.
Furthermore, it should be noted that the $y_{hor}$ is not zero in the case of non-zero tilt angle. But, if we still utilize  $x_{hor}$ to estimate $\theta$ using (\ref{eq:golden1}), there would be some error. Fig. \ref{fig:theta_tilt_error} depicts the RMS value of the error over a range of $\theta$ w.r.t. different values of tilt angle ($\phi$). As seen, for tilt angles in the range of $[-30,30]$, the RMS error in the AOV estimate is decreasing as the absolute value of the tilt increases. This shows that the error induced by the non-zero tilt angle is canceling the approximation error existing in (\ref{eq:golden1}).
The reason is that the absolute value of estimated AOV, i.e. the inverse of $x_{hor}$, is always greater than the absolute value of the actual AOV, i.e. $|\theta_{actual}| < |\theta_{estimated}|=\frac{1}{|x_{hor}|}$. In addition, increasing the absolute value of $\phi$ from $0$ to about $35 \degree$ (i.e. $\phi \in [-30, 30]$) decreases the $|x_{hor}|$ hence increases  $|\theta_{estimated}|$. In essence

\begin{equation}
|\phi| \uparrow \; \Rightarrow \; |x_{hor}| \uparrow \; \Rightarrow \; |\theta_{estimated}| \rightarrow |\theta_{actual}|
\end{equation}
In other words, for a relatively large range of tilt angles, the estimated AOV is still close to the actual value.

\begin{figure}
\center
\includegraphics[scale=0.6]{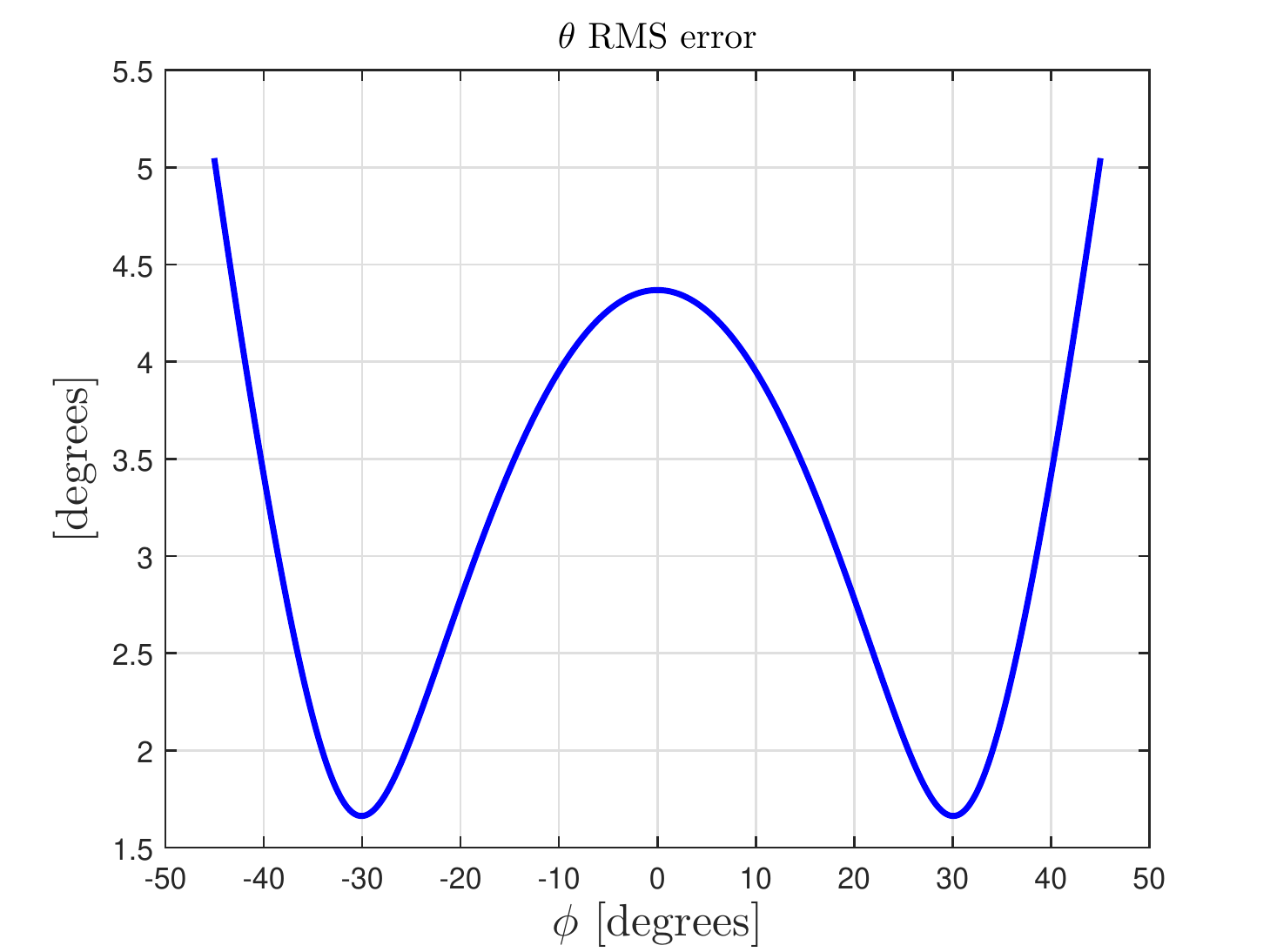}
\caption{RMS value of error in $\theta$ [degrees] when $\theta \in [-45,45]$ versus different values of $\phi$}
\label{fig:theta_tilt_error}
\end{figure}

Fig. \ref{fig:depth_tilt_error} depicts the normalized RMS error in the depth estimate for different values of $\phi$. Normalized error is defined as the error in the depth estimate divided by the actual depth.
As seen, when $\phi \in [-20, 20]$, the RMS value of the normalized error is less than 6\%. 

\begin{figure}
\center
\includegraphics[scale=0.6]{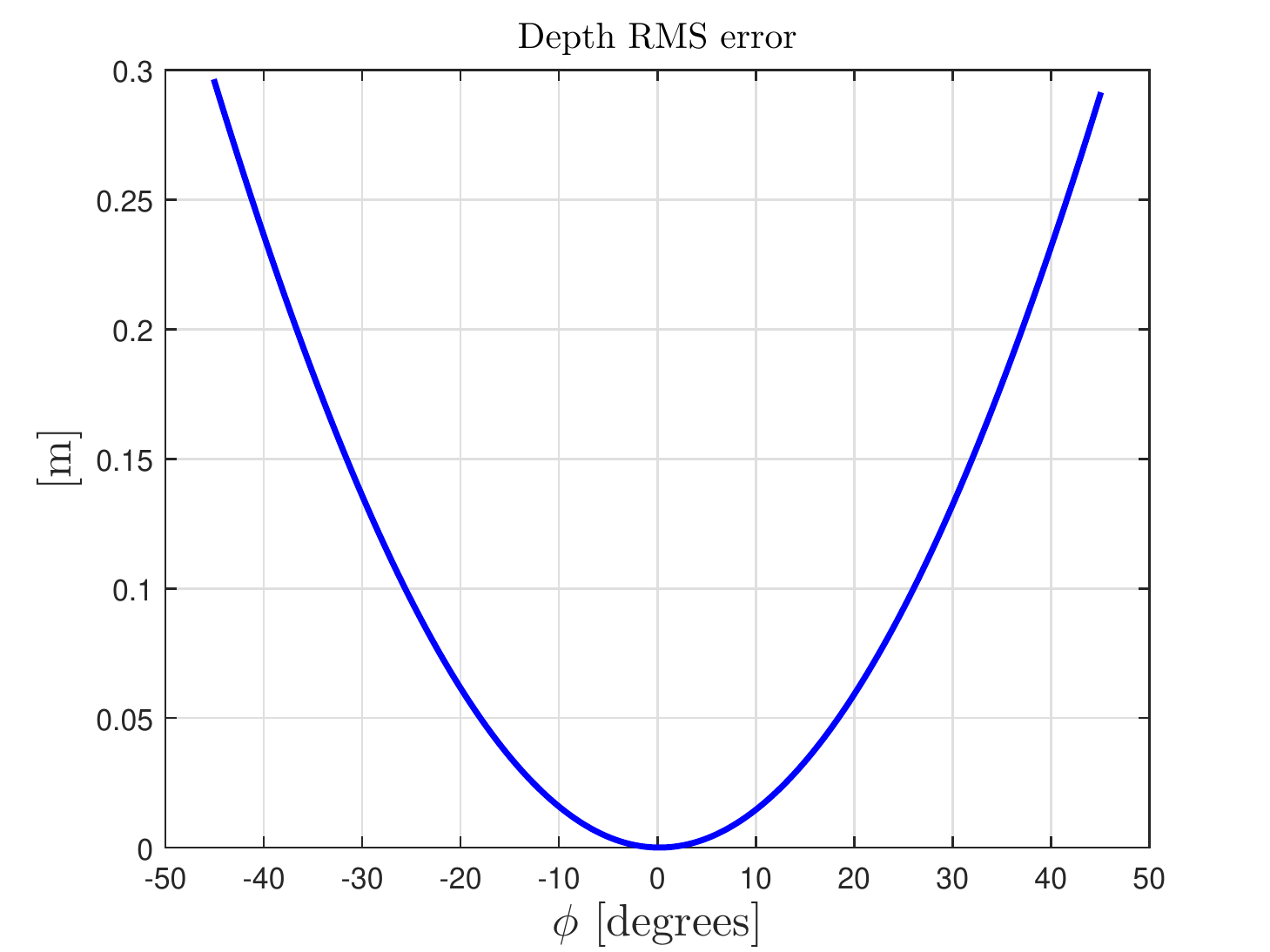}
\caption{RMS value of normalized error in depth [meters] when $\theta \in [-45,45]$ versus different values of $\phi$}
\label{fig:depth_tilt_error}
\end{figure}

In conclusion, if $\phi \in [-20, 20]$, both formulas for AOV and depth estimation are still effective and can be used without any change. 
Furthermore, the mentioned $ [-20, 20]$ is a quite large interval for $\phi$ in practical scenarios. For example, in the university scenario of our experiments, the centroid of the characters region is located at the height of 1.5 m. 
As stated, human users naturally point their phones towards the centroid of the characters plate. Furthermore, they usually hold the phones at a height  approximately equal to the eye level. Hence, non-zero tilt comes from the height difference between the phone and the characters centroid as stated before. In conclusion, if the user is horizontally $1 m$ away from the characters centroid and the phone is held at a height of $113 cm$ to $186 cm$, the tilt varies in the mentioned interval. As user gets further, the tilt range corresponds to even greater range of heights. Hence, the mentioned range is quite large and contains almost all possible people eye heights in indoor scenarios.
We can have a  similar discussion for the airport scenario. In the worst case scenario, assume that the gate centroid is located at a height of $6 m$ and the query camera at $1 m$. In this case, if the user horizontal distance from the centroid is greater than about $14 m$, the tilt angle will be less than $20 \degree$.

\section{Experimental results}
\label{sec:exp}


The localization performance of the proposed system is compared with the state-of-the-art works of Liang's \cite{Liang2013} and Torii's \cite{Torii}.
We compare the methods in terms of the recognition error rate and location estimation error. In Liang's method, the recognition error percentage is defined as the percentage of the  images wrong detected as the best match. In our location recognition problem, wrong best match is a database image that does not contain the same characters as the query.
In Torii's method, the recognition rate is meaningless since there is no image retrieval phase and the best image pair is selected from the database images.
In our method, i.e. OCRAPOSE II,  the wrong recognition corresponds to recognizing one or more characters wrong. In the conventional buildings, wrong recognition of a single character can lead to a false location right besides the query, meters away or even in other floors depending on the location of the mis-recognized character. 
Furthermore, estimation error for all methods is defined as the fine localization error among the truly recognized locations. 

The camera calibration matrices of the phones cameras are assumed known in all scenarios.
Feature extraction and processing needed for the benchmarks is performed using the codes available at \cite{vlfeat}.
In Torii's method, we use the modified VLAD decriptors \cite{allVLAD} for image description to obtain high recognition performance. The modified VLAD descriptor \cite{VLAD} is an improved version of the descriptor suggested in \cite{VLAD} that has demonstrated supreme performance compared to the tf-idf methods such as the one used in \cite{Torii}. 

We studied two university building scenarios and one parking area as the representatives of indoor scenarios. The building scenarios contain a large number of different numbers located at different locations, which makes them appropriate candidates for location recognition error evaluation. On the other hand, the parking scenario only contains one word that can be seen from far distances and different angle-of-views. Hence, it is a suitable scenario for localization (estimation) error evaluation. In addition, two commercial mobile phones, i.e. Google Nexus 4 and Samsung Galaxy S5 were used in experiments to provide device diversity conditions. Furthermore, for both university building scenarios, the $W$ measurement was done only once since all set of characters were of the same width.

\subsection{Scenario I - University building}

In this scenario, there exist light characters printed on dark plates. Sample query images are shown in Fig. \ref{fig:SF_samples}. Furthermore, the floor plan is depicted in Fig. \ref{fig:SF4_plan} and the trace of locations can be seen as a red dotted line.
For the training and query phases of the benchmarks, images are captured at regular locations. That is, for each room number in the Fig. \ref{fig:SF4_plan}, we capture three images, i.e. left, middle and right, with the following (depth, AOV) pair information
\begin{equation}
(1.8 \; m, -33 \degree ), (1.5 \; m, 0 \degree ), (1.8 \; m, +33 \degree ) \nonumber
\end{equation}
This was done at 17 locations and a total of 51 images were taken in this scenario.
We consider odd-index images as queries and even ones as database images. Furthermore, a SAMSUNG Galaxy S5 cellphone is used for image capturing in this scenario.
\begin{figure}[!t]
\centering
\includegraphics[width=.5\textwidth,height=0.3\textheight]{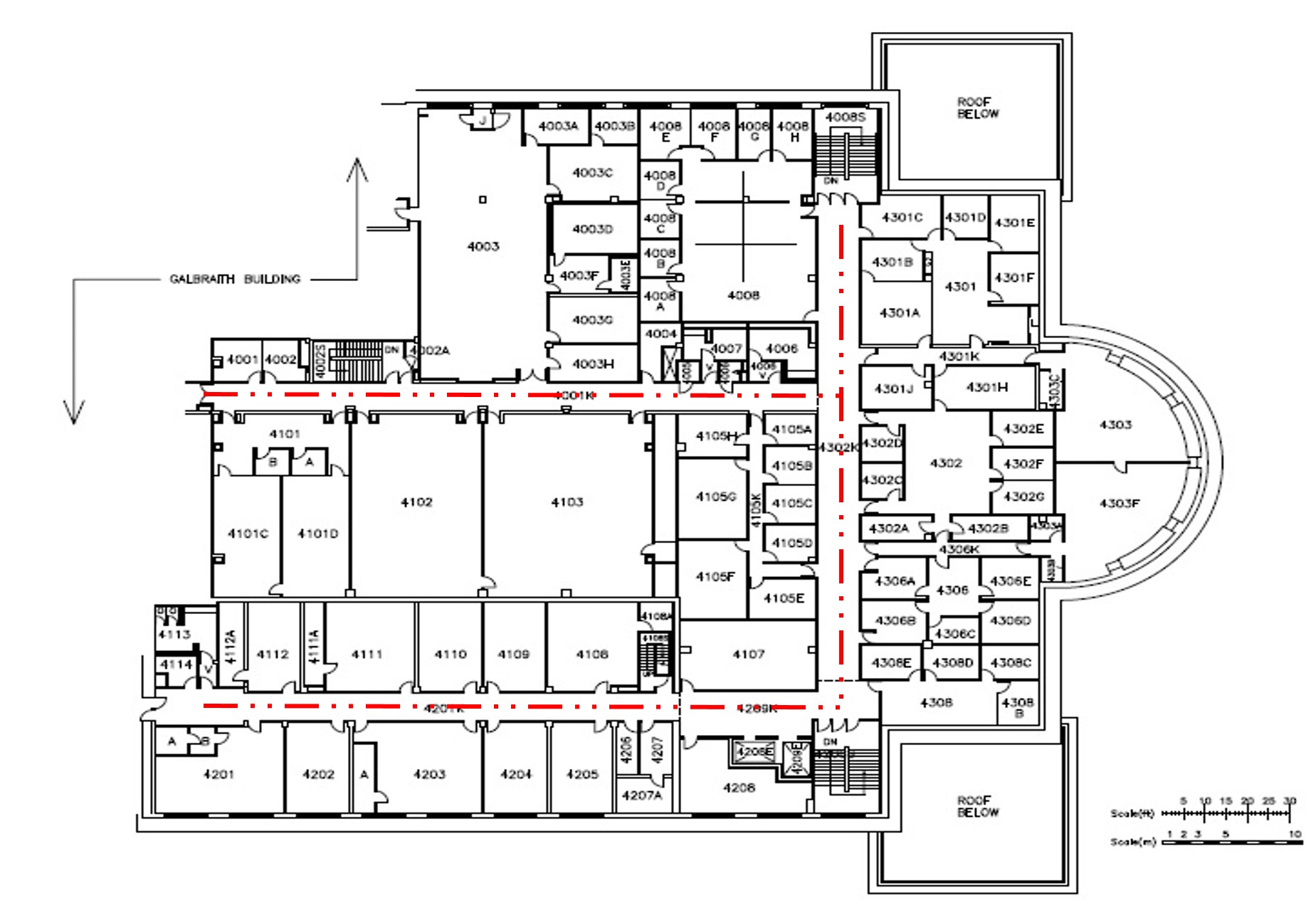}\\
\caption{Scenario I floor plan}
\label{fig:SF4_plan}
\end{figure}
Fig. \ref{fig:SF_samples} shows a number of sample images.
\begin{figure}[!t]
\centering
\includegraphics[width=.45\textwidth,height=0.22\textheight]{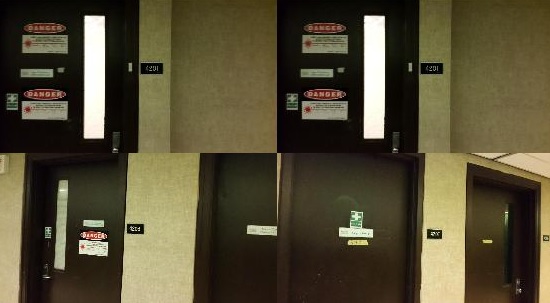}\\
\caption{Scenario I sample images}
\label{fig:SF_samples}
\end{figure}
The location recognition rate (in Liang's method) is about $8 \%$, i.e. 2 correct detections among 51 queries. Hence, we do not include more results from this method in this scenario. 
The low recognition rate demonstrates how weakly image retrieval-based method might perform in the mentioned indoor scenarios.

Table \ref{table: scenario I} compares the results of different methods in terms of recognition rate and mean error of location estimation. 
Furthermore, Fig. \ref{fig:SF_cdf} depicts the CDF (cumulative distribution function) of the localization error for the OCRAPOSE II and Torii's methods. As seen, the OCRAPOSE II demonstrates better localization error compared to that of Torii's.
The reason is Torii's method interpolates the image pair locations hence performs poorly when query is located outside of the database locations, i.e. extrapolation case. However, the proposed method is capable of fine estimation via depth and AOV estimation in all cases including that of extrapolation.

Finally, we downsample the database locations to evaluate the robustness of the methods to coarse databases with fewer locations. In fact, we keep the query points as before and downsample the database locations set. It should be noted that the results of the OCRAPOSE II method are the same for different downsampling factors since it does not rely on the database images. 
Fig. \ref{fig:SF_downsample} shows the average location estimate error for different methods. As seen, the Torii's error is generally increasing with the database downsampling factor. On the contrary, the OCRAPOSE II error is fixed and always less than that of Torii's method.
\begin{table}
\caption{Performance comparison of methods in scenario I}
\begin{center}
    \begin{tabular}{ | l | l | l |}
    \hline
    \bf{Method} & Recognition rate \% & Mean localization error (m)\\ \hline
	Liang's & 7.4 & --- \\ \hline
    Torii's & --- & 0.80 \\ \hline
    OCRAPOSE II & 76 & 0.68 \\ \hline
    \end{tabular}
    \end{center}
    \label{table: scenario I}
\end{table}

\begin{figure}[!t]
\includegraphics[scale=0.6]{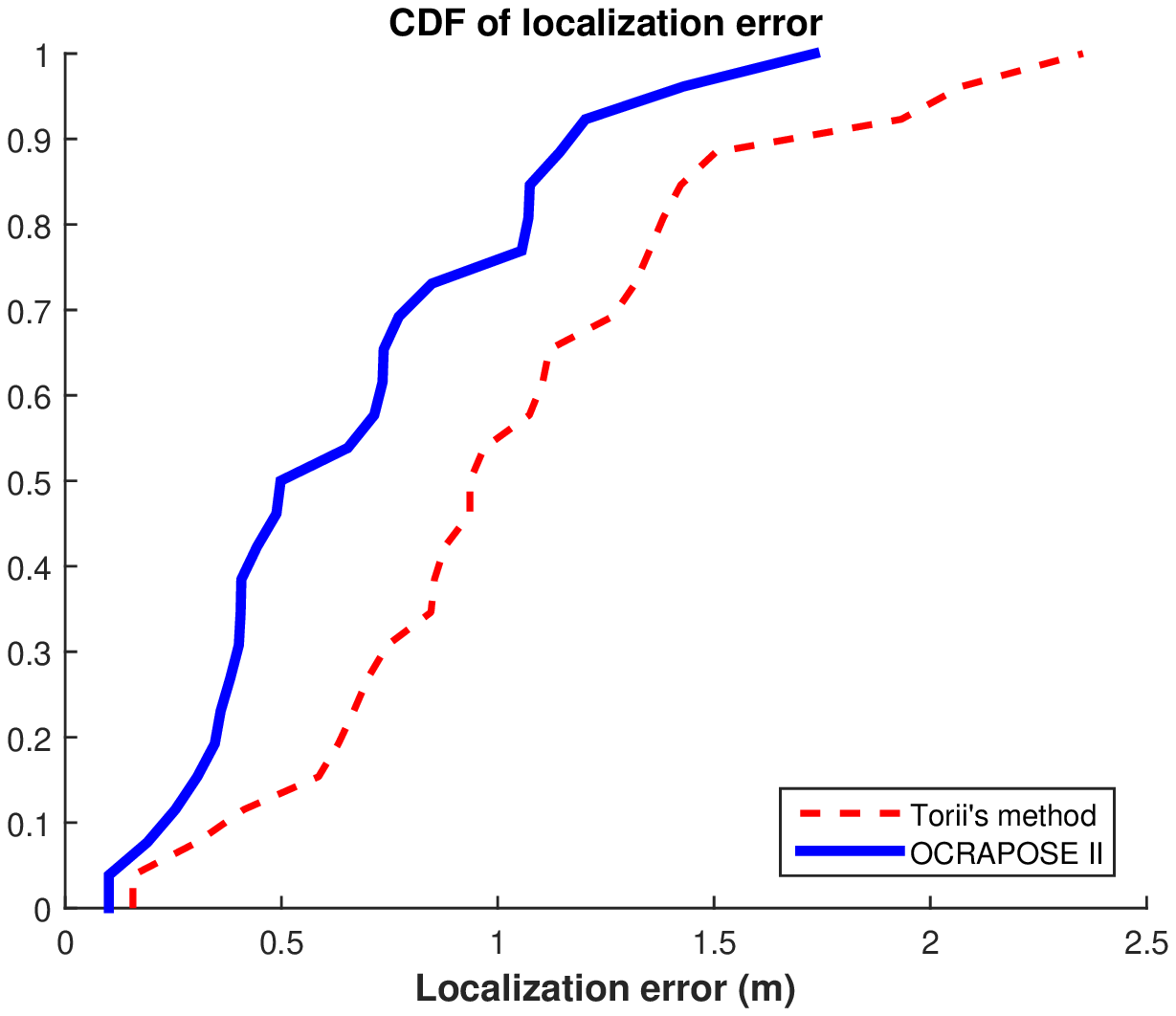}
\caption{location estimation results in scenario I }
\label{fig:SF_cdf}
\end{figure}

\begin{figure}[!t]
\includegraphics[scale=0.6]{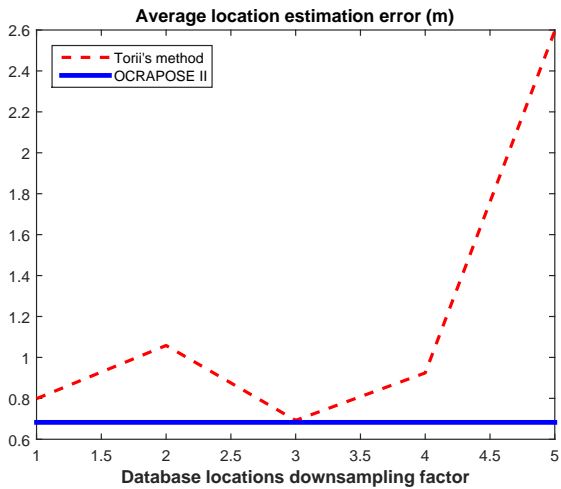}
\caption{Downsampling of the database locations in scenario I}
\label{fig:SF_downsample}
\end{figure}

We also performed the performance comparison in another university building scenario with larger number of locations. In the scenario, we sampled images regularly at 120 locations with a distance of $1$ m using a Google Nexus 4 cellphone.  Depths are in the range of $[1, 2.7] \; m $ and AOVs belong to $[-56\degree, 56 \degree]$. Here, we only show a brief summary of results in Table \ref{table: bahen} in favor of the limited space available. As seen, the location recognition of the proposed system is comparable to Liang's method. But, it demonstrates smaller localization error on average.
The reason is Liang's method only performs location recognition and has no image-based location refinement stage. However, as stated, OCRAPOSE II refines the location estimate in all cases hence outperforms both benchmarks in terms of localization accuracy.

\begin{table}
\caption{Performance comparison of methods in the second university building scenario} 
\begin{center}
    \begin{tabular}{ | l | l | l |}
    \hline
    \bf{Method} & Recognition rate \% &  Mean localization error (m)\\ \hline
    Liang's & 58 & 1.11 \\ \hline
    Torii's & --- & 1.44  \\ \hline
    OCRAPOSE II & 61 & 0.93\\ \hline
    \end{tabular}
    \end{center}
    \label{table: bahen}
\end{table}

\ignore{


In this scenario, there exists dark characters printed on light plates as opposed to the previous one.
We capture one location-tagged image per location for the 120 locations depicted in Fig. \ref{fig:bahen_plan}. Fig. \ref{fig:bahen_samples} shows four sample images taken in this scenario. In fact, we capture images at equal distances of $1~m$ using an ASUS Transformer TF700T tablet.
In this scenario, depths are in the range $[1, 2.7] \; m $ and AOVs belong to $[-56\degree, 56 \degree]$.
\begin{figure}[!t]
\centering
\includegraphics[scale=0.3]{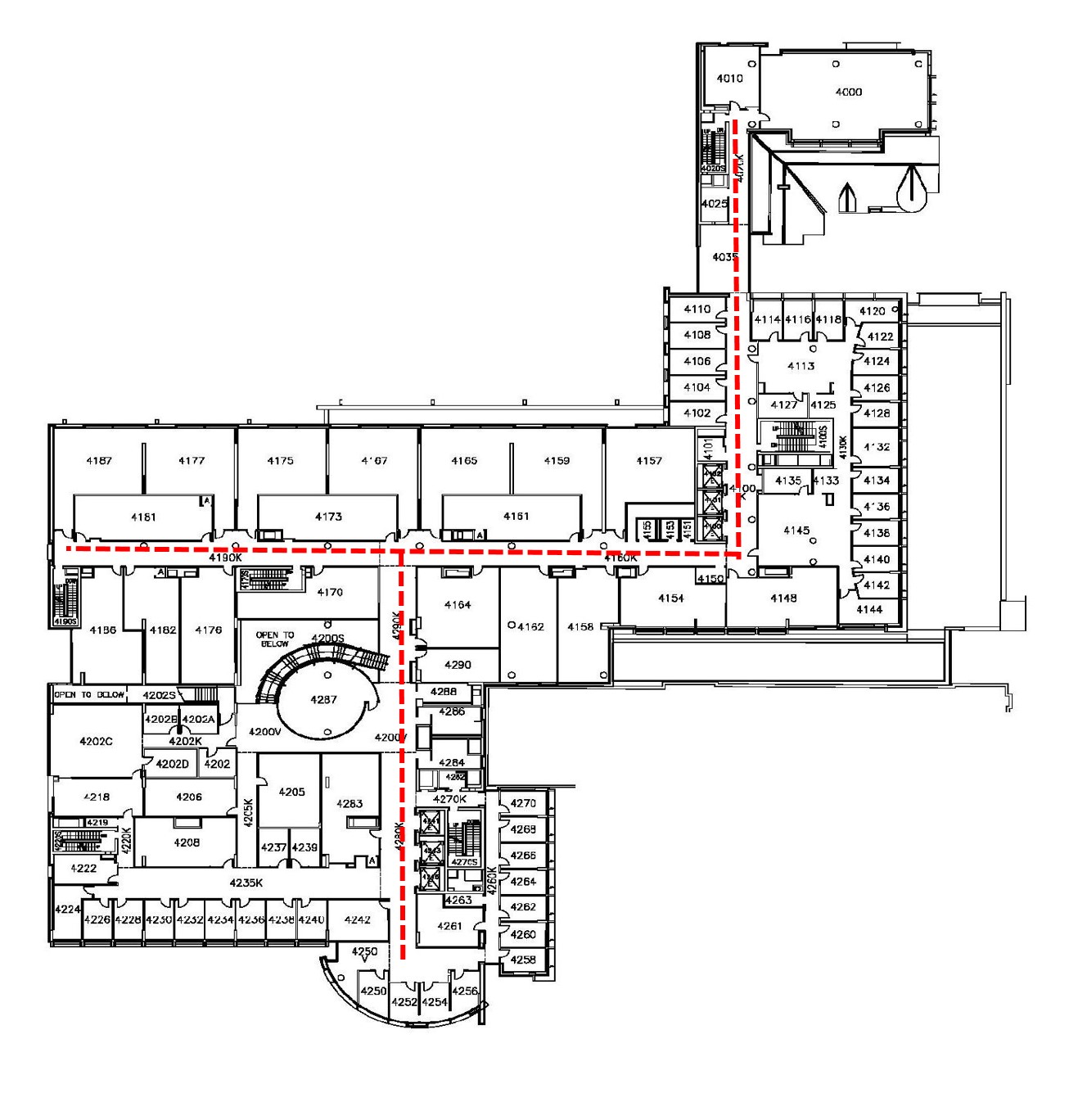}\\
\caption{Scenario II floor plan}
\label{fig:bahen_plan}
\end{figure}
\begin{figure}[!t]
\centering
\includegraphics[scale=0.4]{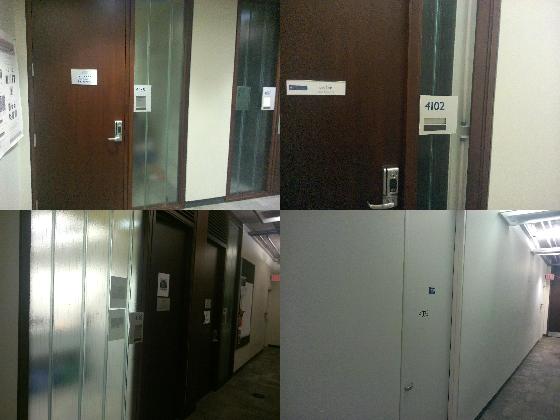}\\
\caption{Scenario II sample images}
\label{fig:bahen_samples}
\end{figure}

We select even indexed images as database and the odd-indexed as queries. In this case, if Liang's method is applied, there will be an offset error of $1~m$ since the closest location in the database is at least $1~m$ away. But, for the OCRAPOSE II and the Torii's methods, there is the possibility of achieving zero localization error.


Fig. \ref{fig:bahen_cdf} shows the CDF of the localization error for the all methods. As seen, the proposed method demonstrates smaller mean/median localization error compared to the benchmarks. 

Table \ref{table: scenario I} lists a summary of location recognition and estimation results for all methods.
As seen, the recognition rate of the proposed method is close to that of Liang's methods.
 It should be noted that the recognition rate of the proposed method depends on the detection and recognition of characters and can be improved by fine tuning and selecting better methods tailored for the scenario of interest. On the contrary, the image retrieval-based methods are performing poorly, mainly due to the repeated scenery present in the mentioned scenarios. In terms of mean error, Liang's method performs very well with zero median error. But, the proposed method performs as good as Torii's method.

\begin{table}
\caption{Performance comparison of methods in scenario I} 
\begin{center}
    \begin{tabular}{ | l | l | l |}
    \hline
    \bf{Method} & Location recognition rate \% &  Mean localization error (m)\\ \hline
    Liang's & 58 & 1.11 \\ \hline
    Torii's & --- & 1.44  \\ \hline
    OCRAPOSE II & 61 & 0.93\\ \hline
    \end{tabular}
    \end{center}
    \label{table: scenario I}
\end{table}

\begin{figure}[!t]
\centering
\includegraphics[scale=0.6]{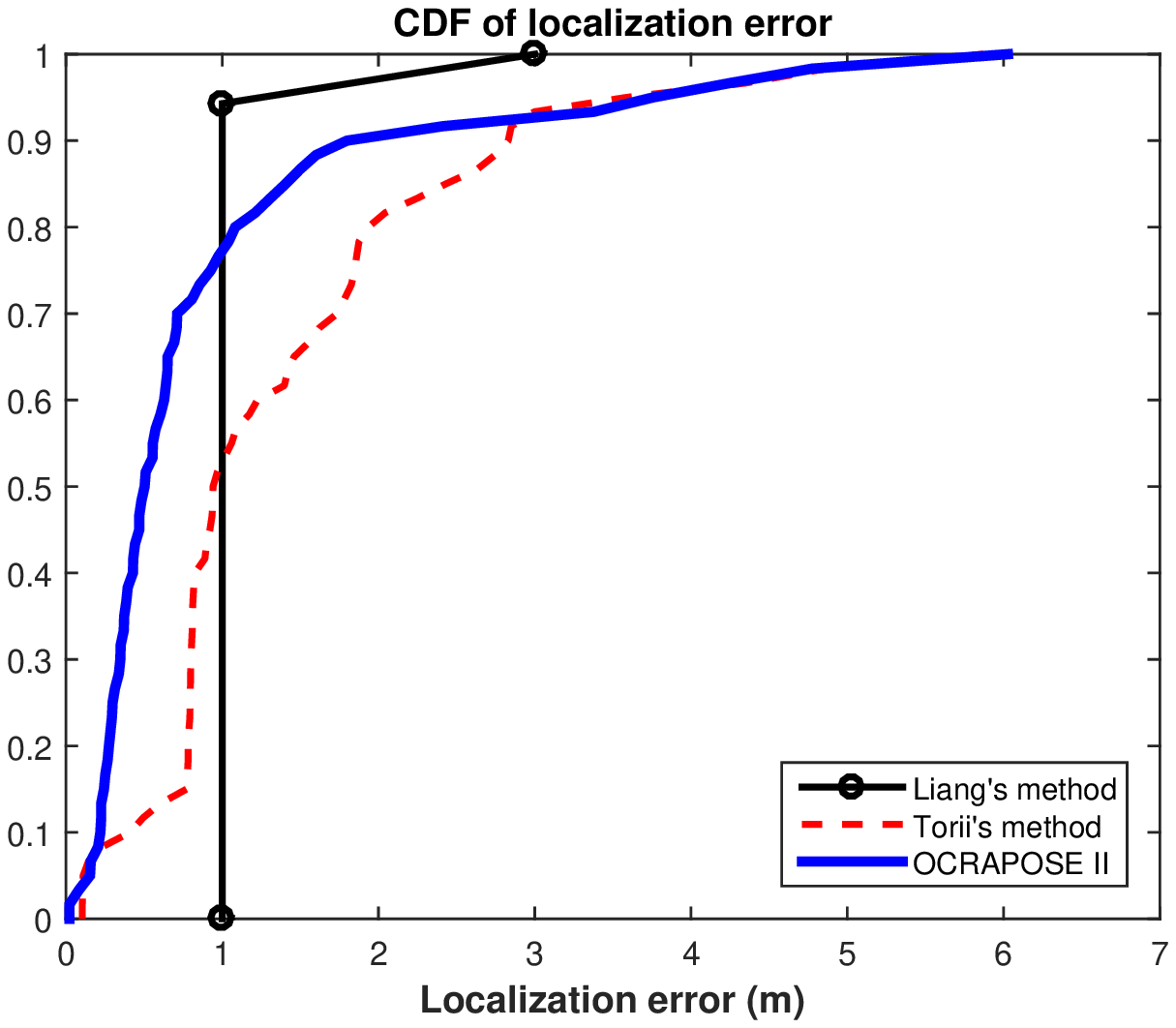}\\
\caption{Location estimation results in scenario I}
\label{fig:bahen_cdf}
\end{figure}

\begin{figure}[!t]
\centering
\includegraphics[scale=0.6]{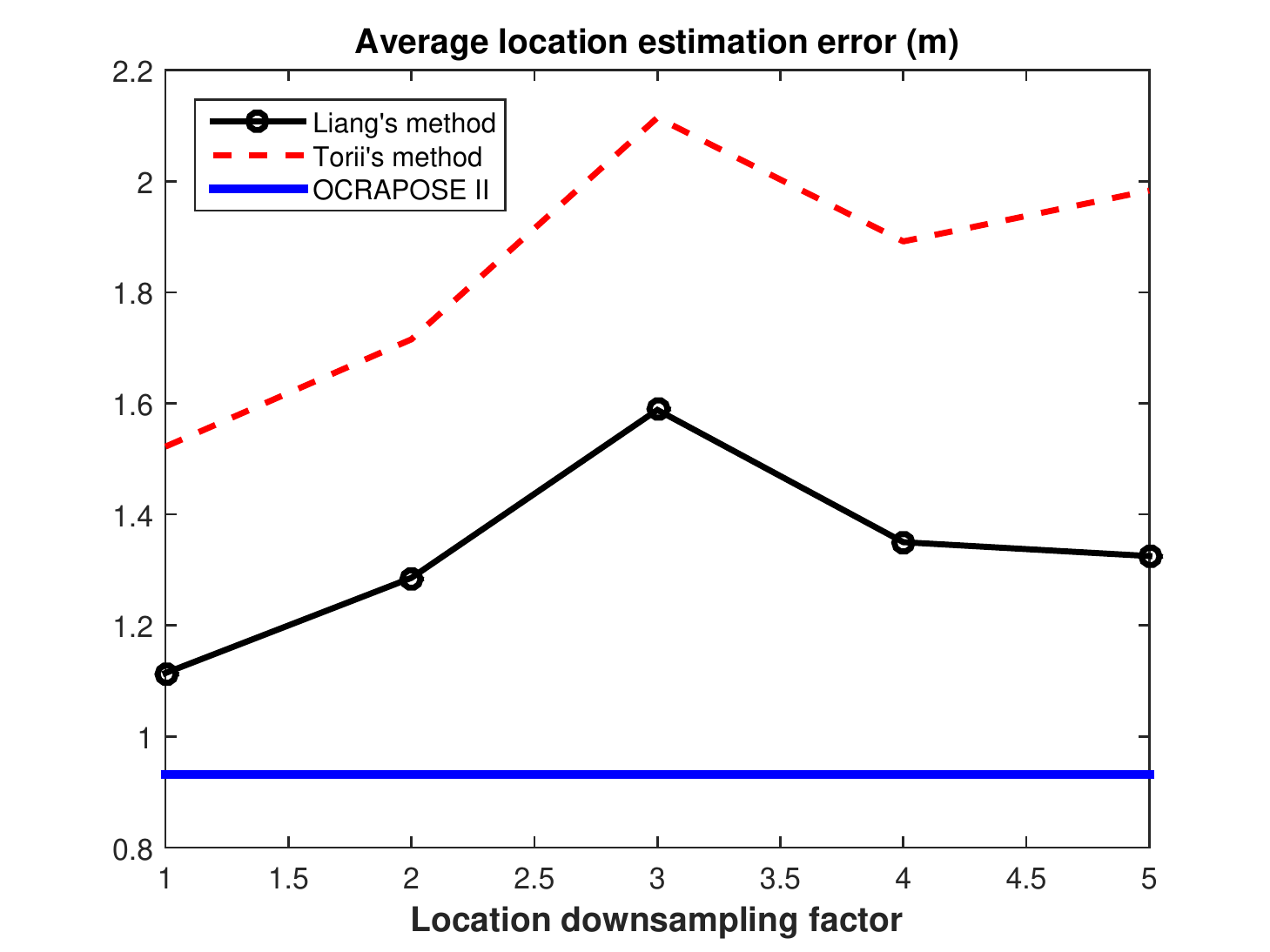}\\
\caption{Downsampling of the database locations in scenario II}
\label{fig:bahen_downsample}
\end{figure}

}

\subsection{Scenario II - Parking}
As stated before, this scenario is designed to compare the location estimation error of the methods only.
In fact, only the word \emph{East}, as depicted in Fig \ref{fig:east}, is used for localization. The actual width ($W$) of the word is $95.4$ cm.
Fig. \ref{fig:east_samples} depicts sample images taken in this scenario. 
Images are taken from three angle-of-views of $0$, $+45 \degree$ and $-45 \degree$. Depths belong to the interval of $[1, 40] ~m$.

Table \ref{table: scenario II} shows a summary of the location estimation error results. As seen, the OCRAPOSE II is outperforming benchmarks in terms of location accuracy with a large gap. It is due to existence of a larger OCR box in this scenario compared to previous ones. It leads to lower relative error in the estimation of the box width. Fig. \ref{fig:parking_cdf} is comparing the localization error results. As seen, OCRAPOSE II demonstrates smaller average localization errors.
 Moreover, Fig. \ref{fig:parking_downsample} shows the effect of database locations downsampling on the average localization error. As depicted, the OCRAPOSE II demonstrates much less error compared to the benchmarks.

The studied scenarios demonstrate the applicability of the proposed OCR-based method. The assumptions we made regarding the alignment of characters and  relative position and orientation of the user with respect to the characters usually hold in practice. Furthermore, the size of the conventional characters is large enough so that they can be recognized with high probability even at the furthest possible locations of the user. In fact, texts and gate numbers are bigger in larger areas such as airports, parkings, etc., where user might be tens of meters away. Based on our experiments, each character is seen at least with a width of $10 \sim 15$ pixels in the query image, which is large enough for the OCR engines to recognize.

\begin{figure}[!t]
\centering
\includegraphics[scale=0.3]{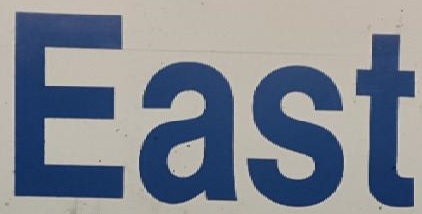}\\
\caption{The characters present in the scenario II}
\label{fig:east}
\end{figure}

\begin{figure}[!t]
\centering
\includegraphics[scale=0.6]{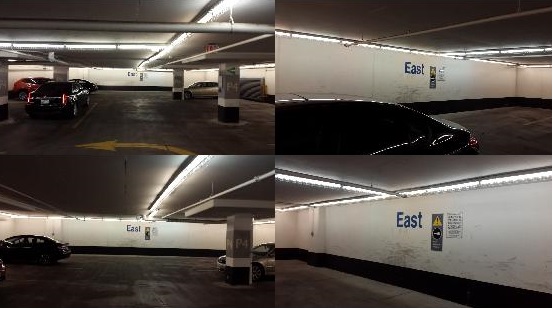}\\
\caption{Sample images from the scenario II}
\label{fig:east_samples}
\end{figure}

\begin{table}
\caption{Performance comparison of methods in scenario II} 
\begin{center}
    \begin{tabular}{ | l | l | l |}
    \hline
    \bf{Method} & Mean localization error (m)\\ \hline
    Liang's & 2.08 \\ \hline
    Torii's & 1.47  \\ \hline
    OCRAPOSE II & 0.51\\ \hline
    \end{tabular}
    \end{center}
    \label{table: scenario II}
\end{table}

\begin{figure}[!t]
\centering
\includegraphics[scale=0.6]{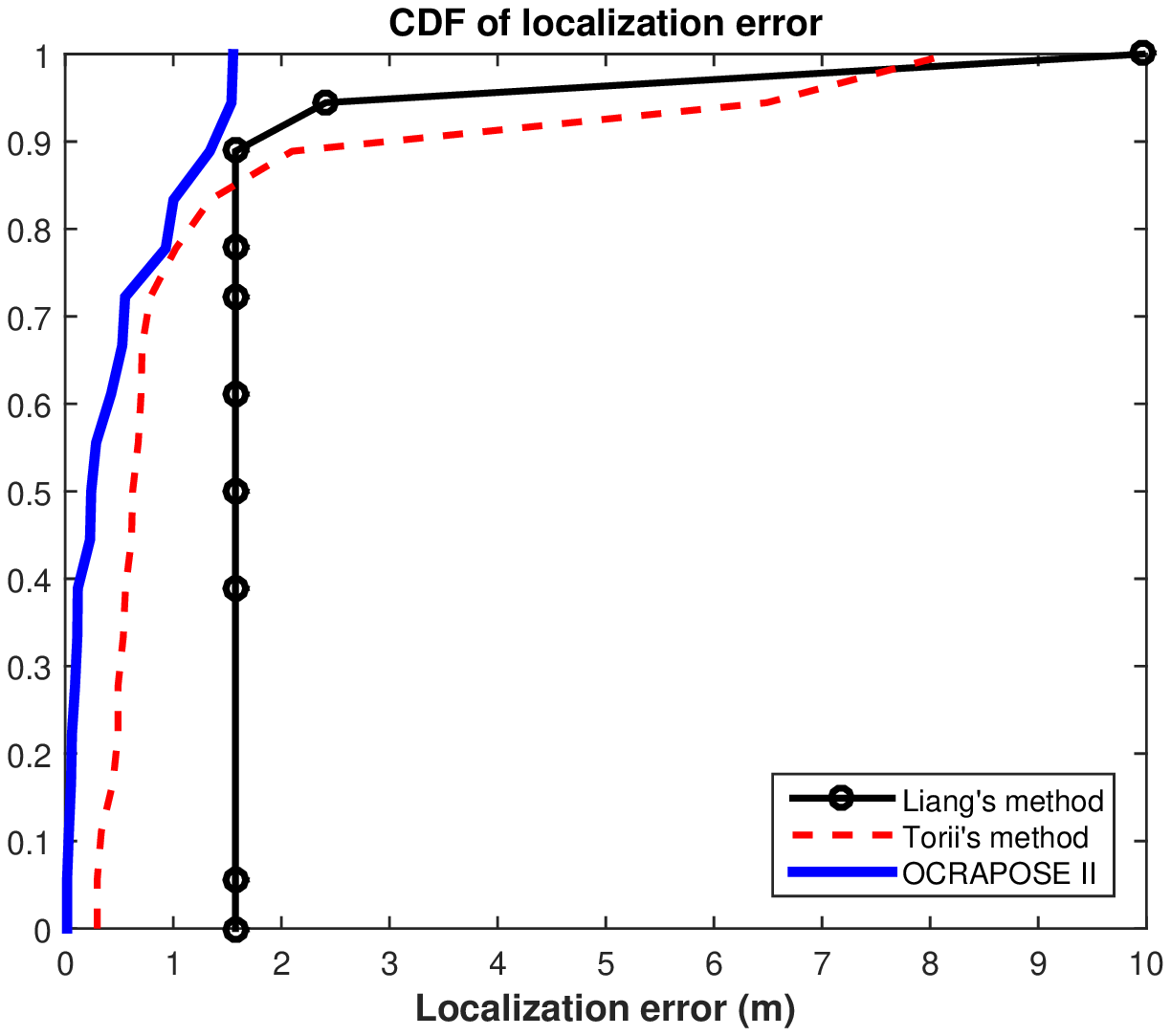}\\
\caption{location estimation results in scenario II}
\label{fig:parking_cdf}
\end{figure}

\begin{figure}[!t]
\centering
\includegraphics[scale=0.6]{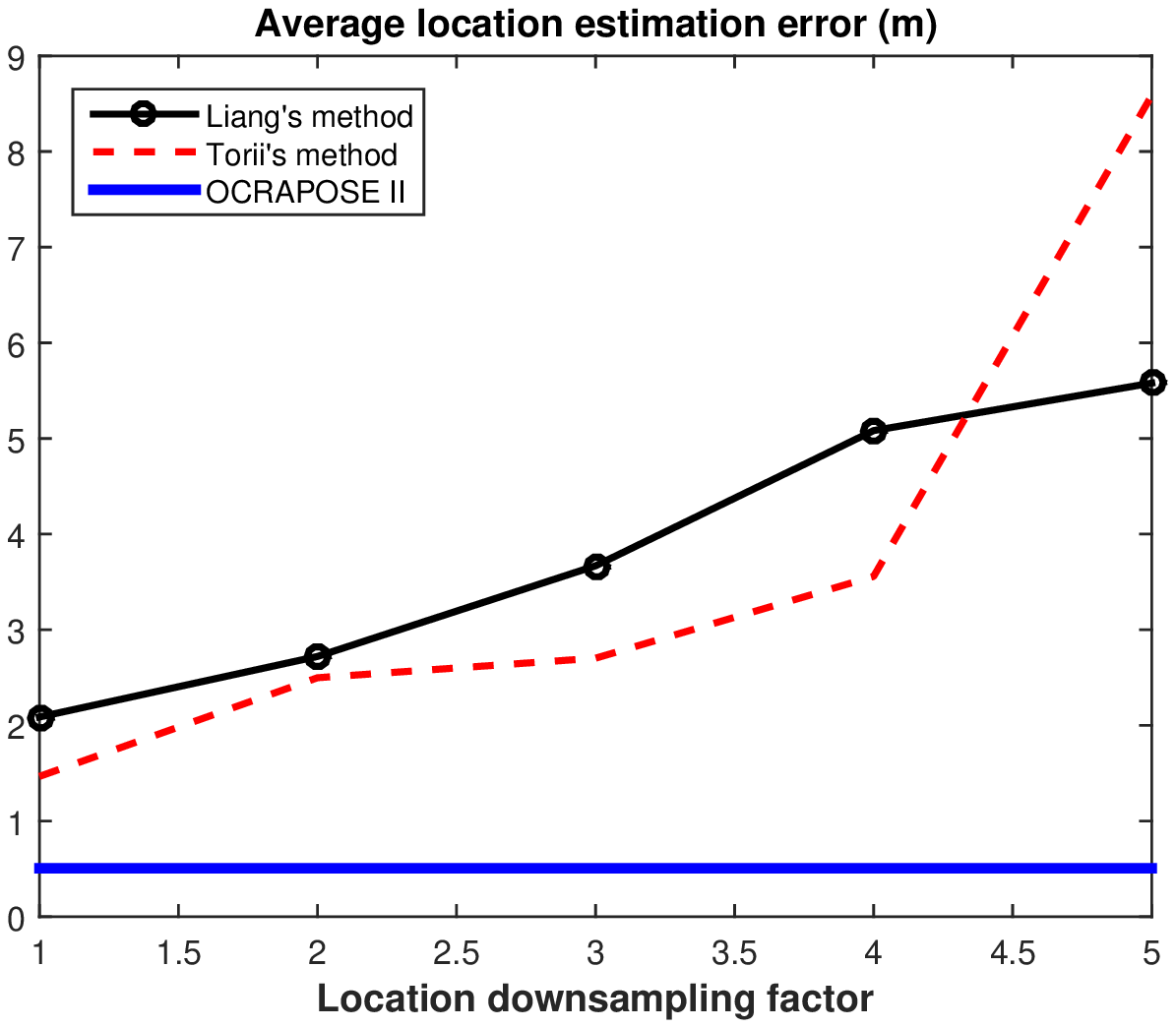}\\
\caption{Downsampling of the database locations in scenario II}
\label{fig:parking_downsample}
\end{figure}

\section{Conclusion}
\label{sec:conclusion}

In this paper, we discuss a number of indoor scenarios, which are challenging for the existing localization methods and propose using OCR to recognize characters as suitable location distinctive landmarks. A novel system is proposed that utilizes OCR to perform rough localization. 
Two novel formulas are also proposed for angle-of-view and depth estimation, which are used to refine the location estimate.
Our experiments demonstrate that the proposed OCR-based system achieves better performance compared with the state-of-the-art localization methods in terms of location recognition rate and average localization error. It is also shown that benchmarks performance degrades as database locations set becomes sparser, while the performance of the proposed system is independent of the database sparsity and remains constant.

\bibliographystyle{IEEEtran}
\bibliography{journal}

\ignore{
\ifCLASSOPTIONcompsoc
  \section*{Acknowledgments}
\else
  \section*{Acknowledgment}
\fi

The authors would like to thank...
}

\ifCLASSOPTIONcaptionsoff
  \newpage
\fi

\ignore{
\begin{IEEEbiography}{Michael Shell}
Biography text here.
\end{IEEEbiography}

\begin{IEEEbiographynophoto}{John Doe}
Biography text here.
\end{IEEEbiographynophoto}


\begin{IEEEbiographynophoto}{Jane Doe}
Biography text here.
\end{IEEEbiographynophoto}
}




\end{document}